\documentclass{article}

\usepackage{arxiv_upload}

\usepackage{amsmath}
\usepackage{amsthm}
\usepackage{graphicx}
\usepackage[utf8]{inputenc} % allow utf-8 input
\usepackage[T1]{fontenc}    % use 8-bit T1 fonts
\usepackage{hyperref}       % hyperlinks
\usepackage{url}            % simple URL typesetting
\usepackage{booktabs}       % professional-quality tables
\usepackage{amsfonts}       % blackboard math symbols
\usepackage{nicefrac}       % compact symbols for 1/2, etc.
\usepackage{microtype}      % microtypography
\usepackage{xcolor}         % colors

\newtheorem{theorem}{Theorem}[section]
\newtheorem{proposition}[theorem]{Proposition}
\newtheorem{lemma}[theorem]{Lemma}

\newtheorem{definition}[theorem]{Definition}
\newtheorem{assumption}[theorem]{Assumption}
\newtheorem{remark}[theorem]{Remark}

\title{When Does In-Context Search Help? A Sampling-Complexity Theory of Reflection-Driven Reasoning}

\author{Yotam Wolf, Noam Wies, Amnon Shashua\\The Hebrew University\\\{yotamwolf,noam.wies,shashua\}@cs.huji.ac.il
}

\begin{document}

\maketitle

\begin{abstract}
Training large language models (LLMs) with extended reasoning has enabled \textit{in-context search}, in which models iteratively generate, critique, and revise solution attempts. We provide a theoretical analysis of in-context search by modeling it as approximate inference over reasoning traces, where the base model defines a prior and self-reflection provides feedback for posterior updates, and study the resulting inference-time sampling complexity - the number of sequential attempts needed to achieve high success probability. We show that when reflections reliably localize early mistakes, in-context search can yield exponential improvements over the base model, solving problems with exponentially small zero-shot pass rates using only a polynomial number of sequential attempts, whereas when this property fails, conditioning on past attempts offers no asymptotic benefit over parallel sampling. We further show that these gains are robust and learnable: approximate posterior updates suffice, and cross-entropy training on search rollouts recovers the required behavior with polynomial sample complexity. Finally, we show that under a stagewise abstraction of reinforcement learning with verifiable rewards, the optimal policy extension implements the same posterior reweighting rule. We validate key qualitative predictions of the theory on real large reasoning models.
\end{abstract}

\section{Introduction}

Chain-of-thought (CoT) reasoning represents an LLM's intermediate reasoning steps as a sequence connecting the input to the final answer, and was originally used to decompose complex tasks into more manageable subtasks \citep{wei2022chain}. Subsequent methods such as tree of thought, self-consistency, and reflection extended this paradigm by exploring multiple reasoning paths, evaluating alternatives, and iteratively refining conclusions \citep{yao2023tree,wang2022self,shinn2023reflexion}, leading to the emergence of large reasoning models (LRMs) that solve problems through sequential revisions and multiple solution attempts, a process known as ``in-context search'' \citep{jaech2024openai,chen2021evaluating,li2022competition,AlphaCode2T,ridnik2024code}. This paradigm has enabled high accuracy on reasoning problems under single-sample inference, outperforming earlier LLMs that relied on extensive multi-sample aggregation \citep{chollet2024openai}, and has been prominently observed in models trained with reinforcement learning with verifiable rewards (RLVR), such as DeepSeek-R1 \citep{guo2025deepseek} and Kimi-k1.5 \citep{team2025kimi}, where search-like reasoning within the context window emerges from training. It has since motivated a growing body of work on reinforcement-based reasoning and search behavior \citep{yeo2025demystifying,wang2025information,liu2025attention,yu2025demystifying}.

Despite these successes, it remains unclear when in-context search provides genuine advantages over parallel sampling from a base model, and when it does not. We study this question through the lens of inference-time sampling complexity: how many sequential attempts (or backtracks) are required to achieve high success probability, compared to parallel sampling from the base model. The recent theoretical work of \cite{shalev2025reasoning} demonstrated exponential sample efficiency using an explicit tree-search procedure that prunes failed branches by removing them from context. In contrast, modern LRMs typically retain the full history of attempts in-context, raising the question of when this unstructured form can match the same guarantees.

In this work, we address this question by theoretically characterizing the search process underlying in-context search. Building on the search-based perspective of \cite{shalev2025reasoning}, we model reasoning as inference over reasoning traces, focusing on the regime relevant to modern LRMs, where the full history of attempts - including failures - remains in the context. Our framework (Section \ref{sec:framework}) formalizes in-context search as an inference process with three components: (i) a prior over reasoning traces induced by the base model, (ii) a reflection mechanism that localizes errors in attempted solutions, and (iii) posterior updates that reweight candidate continuations based on the accumulated reflection feedback. See illustration in figure \ref{fig:heuristic}.

Using this framework, we identify reflection
as the key determinant of inference-time sampling complexity. Under local updates, negative feedback helps only when it identifies prefixes whose continuations should be suppressed. If reflections localize errors only at late stages, in-context search offers no asymptotic improvement over independent sampling and can even degrade performance (Propositions~\ref{res:late_reflections} and~\ref{res:positive_reflections}). In contrast, when reflections reliably identify early errors, in-context search can achieve exponential gains: problems whose base-model pass rate decays exponentially with reasoning depth can be solved with arbitrarily high probability using only polynomially many sequential attempts (Theorem~\ref{result:inference_backtracks}).

We further show that this posterior-update behavior is robust and learnable. Approximate step-wise updates suffice to preserve the inference-time improvement, and standard cross-entropy training on efficient search rollouts yields polynomial sample complexity. We also analyze a stagewise abstraction of RLVR and show that, under noisy but mostly correct reflection feedback, the optimal extension of the policy - corresponding to generating additional reasoning attempts conditioned on past attempts and reflections - implements the same posterior reweighting, providing a mechanism for the emergence of search-like reasoning in practice.

Finally, in Section~\ref{sec:experiments}, we test qualitative predictions of the theory. First, we evaluate the reflection assumption on synthetic reasoning traces with injected errors, measuring whether a model can localize the earliest incorrect step. Second, on AIME 2025, we study real LRMs through prefix-conditioned pass rates, an observable proxy for posterior mass on correct solutions. Successful trajectories generally increase in downstream pass rate, unsuccessful trajectories exhibit low and oscillatory behavior, and self-correction segments produce both negative and large positive changes, consistent with inference-like redistribution rather than monotone accumulation of useful computation.

To recap our contributions:
\begin{itemize}
    \item We characterize when reflection-driven inference reduces sampling complexity from exponential (parallel sampling) to polynomial (sequential attempts), and when it does not.
    \item We show that the posterior update rule underlying efficient in-context search is both learnable from rollouts with polynomial sample complexity and arises as the optimal policy in a stagewise extension setting, providing a mechanism for its emergence under RLVR-like training.
\end{itemize}

\begin{figure}[h]
    \centering
    \includegraphics[width=0.7\linewidth]{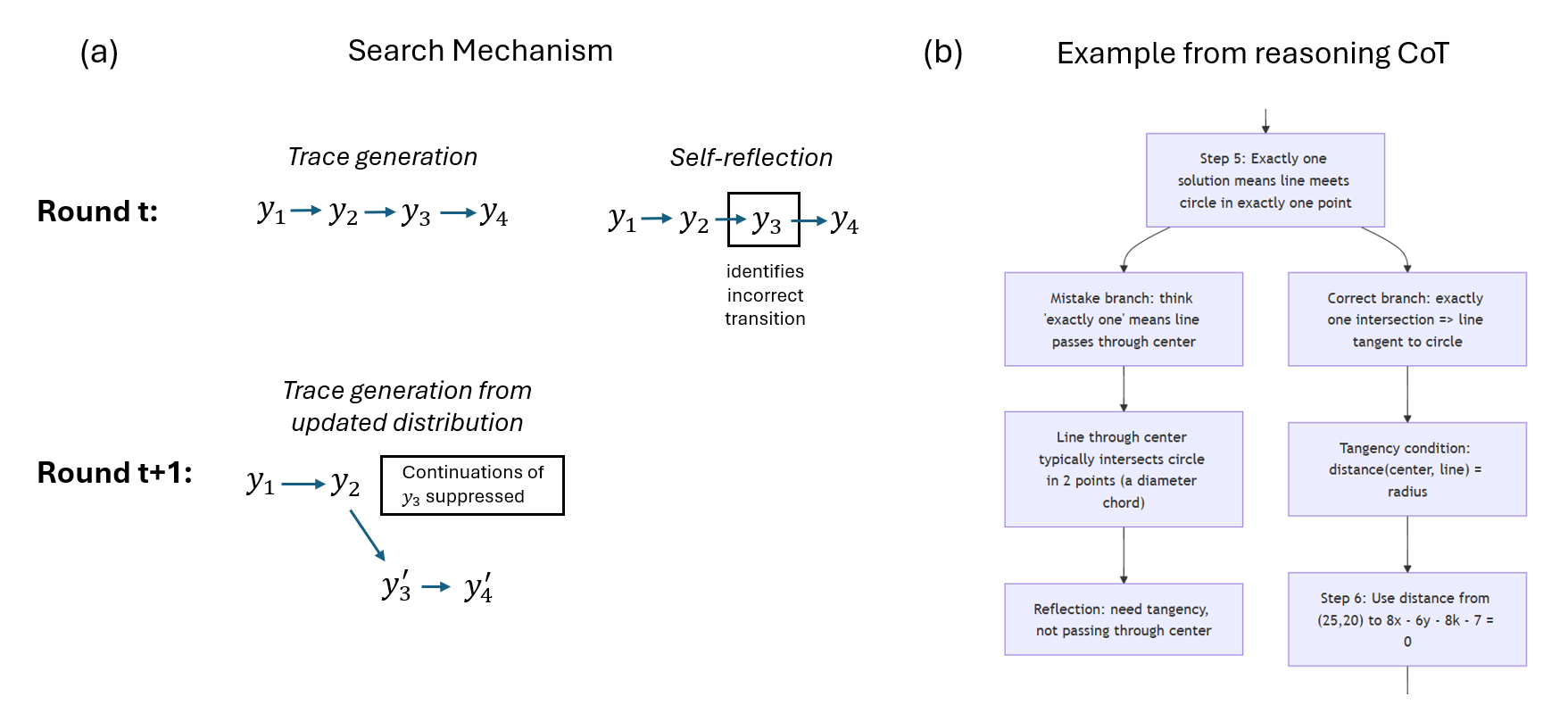}
    \caption{(a) Visualization of search heuristic. On the first attempt the model samples a path based on its prior likelihood. At the end of the attempt, receives a reflection indicating the last node is incorrect. A posterior update takes place, and on the second attempt the model samples accordingly. (b) Example of subtree in a real LRM-generated CoT.}
    \label{fig:heuristic}
\end{figure}

\section{Related Work}

\paragraph{LLMs empowered by CoT:} Prior theoretical work on chain-of-thought has focused on expressivity and learnability, in terms of computational hardness of the problems, without consideration of the self-supervised search process performed by models in practice \citep{ahuja2024provable,wies2022sub,merrill2023expresssive,malach2023auto,peng2024limitations,xu2024large}. More recent work studies CoT from a search perspective, modeling reasoning as traversal over a tree of possible continuations with an explicit search mechanism \citep{shalev2025reasoning,xia2025rethinking}. In contrast, we focus on the regime relevant to modern LRMs, where failed attempts are retained in context and search proceeds through implicit updates rather than explicit pruning, and characterize the resulting inference-time efficiency of such processes.

\paragraph{Using LLMs as priors:} Several works treat LLMs as probabilistic priors for reasoning, performing explicit Bayesian inference or analyzing their probabilistic reasoning capabilities \citep{qiu2025bayesian,zhang2025beyond,pournemat2025reasoning}. Closely related is \cite{karan2025reasoning}, which performs Markov chain Monte Carlo sampling over LLM-induced distributions. We adopt a similar perspective, viewing reasoning as inference over a base-model prior, but focus on the inference-time sample complexity of such procedures. In contrast to prior work, we characterize when sequential conditioning yields exponential improvements over parallel sampling.

\paragraph{In-context search:} Recent empirical work has studied how to induce and improve search behavior in LLMs, focusing on training procedures, exploration strategies, and test-time methods for multi-step reasoning \citep{yeo2025demystifying,snell2024scaling,singhimproving,wang2025information,kim2025meta,liu2025attention}. In contrast, these works do not characterize when such search is provably more efficient than parallel sampling. A central component of our analysis is self-reflection, which we model as noisy feedback that localizes errors and drives inference updates; our theory characterizes when such feedback is sufficient to yield efficient inference-time search. This mechanism can be further enhanced through improved self-awareness \citep{kim2025meta} or integration with external tools \citep{yu2025demystifying}.

\section{In-Context Search as Sequential Inference}
\label{sec:framework}

We study when sequential reasoning - where a model generates multiple attempts conditioned on past failures - can outperform independent sampling from a base language model. We model this as an inference process over reasoning traces: (i) the base model defines a prior over possible solutions, (ii) reflection provides information about errors in sampled trajectories, (iii) conditioning on this feedback induces an implicit posterior update.

The search space is motivated by the work of \cite{shalev2025reasoning}, where reasoning is represented at the level of semantic steps rather than tokens, and treats prefixes of reasoning traces as the states of the process. 
While their work represents this space as a tree and performs explicit search by pruning failed branches, we focus on the regime relevant to modern LRMs, where failed attempts are retained in context and influence future generation, inducing inference through conditioning.

\subsection{Base Model as a Prior Over Reasoning Traces}

We model the base language model as an autoregressive distribution $P_\theta(\cdot \mid x)$ over reasoning traces. A \emph{reasoning trace} is a sequence 
\[
y = (y_1, \dots, y_n),
\]
corresponding to a single solution attempt, where each $y_i$ denotes a reasoning step, and
\[
P_\theta(y \mid x) = \prod_{i=1}^n P_\theta(y_i \mid x,y_{<i}).
\]

We treat prefixes of reasoning traces as the states of the inference process. For a trace $y$, let
\[
h_i = y_{1:i}
\]
denote the prefix up to step $i$, with $h_0$ the empty prefix. Generation corresponds to iteratively extending a prefix $h_i$ by sampling the next step $y_{i+1} \sim P_\theta(\cdot \mid x,h_i)$.

For notational convenience, these prefixes can be organized into a rooted tree $T_x = (V,E)$, where each node $v \in V$ is a prefix and an edge $(v,v')$ exists if $v'$ extends $v$ by one step with nonzero probability under $P_\theta(\cdot \mid x)$. A complete reasoning trace corresponds to a root-to-leaf path. We assume depth at most $n$ and branching factor at most $W$.

To ensure that correct solutions are reachable under the base model - so that inference is non-trivial - we impose a local support condition.

\begin{assumption}[$\gamma$-local support]\label{assumption:universality}
Let $h$ be any prefix that lies on a correct reasoning trace. Then under
$P_\theta(\cdot \mid x,h)$, the probability of sampling a next step that
extends $h$ along a correct reasoning trace is at least $\gamma > 0$.
\end{assumption}

This is a one-step support condition: the model assigns non-negligible probability to a correct next transition.
We note this assumption is mild as the dependence on $\gamma$ will be logarithmic, and can be well below any standard sampling scheme.
Under this view, the base model defines a prior over reasoning traces, and independent sampling corresponds to drawing trajectories from $P_\theta(\cdot \mid x)$, consistent with the importance of the base model for search success \cite{yue2025does}.

The performance of this baseline is given by the pass rate $\Pr_{y \sim P_\theta(\cdot \mid x)}[y \in \mathcal{Y}_{\mathrm{correct}}(x)]$, which in the worst case decays exponentially with the depth $n$. In particular, if at each step the probability of selecting a correct continuation is at most $p<1$, then the full trajectory is at most $p^n=O(\exp(-n))$.

\subsection{Reflection as a Source of Information}

For an iterative sampling procedure to beat the efficiency of sampling from the prior, an inference process must be performed over failed attempts. But simply ruling out an incorrect reasoning trace does not narrow the search space significantly. In contrast, ruling out a prefix prunes all of its continuations and narrows the search space significantly. In our framework such information is obtained through self-reflection, where the model observes the full context and the latest failed attempt, and estimates the source of the failure. We model the process in rounds $t=1,2,...$ of sequential attempts, with $\mathcal{C}_t$ denoting the in-context history.

\begin{definition}[Adaptive reflection signal]
\label{def:reflection}
Let \(\mathcal{C}_{t-1}\) denote the history up to round \(t-1\). At round \(t\),
suppose the sampled reasoning trace \(\omega_t=(y_1,\dots,y_n)\) is incorrect. Let
\(k\) denote the first invalid transition along the trace: the smallest
index such that the prefix \(y_{1:k}\) does not lie on any correct reasoning
trace.

A reflection signal identifies a transition
$(h_{t},a_{t})=
(y_{<r_t}, y_{r_t})
$
as invalid. Define
\begin{equation}
Z_t=\mathbf{1}\{r_t=k\}.
\end{equation}
We assume that for some \(p_r\in(0,1]\) and \(m_r \ge 0\), if the earliest incorrect prefix
\(y_{1:k}\) was used in at least \(m_r\) previous failed attempts in
\(\mathcal{C}_{t-1}\), then
\begin{equation}
\mathbb{E}[Z_t\mid \mathcal{C}_{t-1}]\ge p_r.
\end{equation}
When \(r_t\neq k\), the identified transition is a later invalid transition on the same trace.
\end{definition}

Informally, this assumption requires that once the model encounters the same failure mode multiple times, it can identify the earliest incorrect step with non-negligible probability.
The reflection performs a diagnostic task - identifying where a reasoning trace becomes incorrect - rather than constructing a correct solution from an exponentially large set of candidates. The probability of identifying the earliest incorrect step \(p_r\) need not be large, and we require only that this signal becomes reliable after repeated exposure to the same failure mode, captured by the threshold \(m_r\).

In our framework, the model's new attempts can restart from previously visited states, as in practice. Accordingly, reflection can be applied to the full reasoning trace, including both reused prefixes and newly generated steps. This is important: if reflection were restricted to only the newly generated suffix, the model could repeatedly restart from an already invalid prefix and enter a reasoning loop.
Appendix~\ref{sec:loops} provides empirical evidence for the necessity of this assumption.

\begin{remark}[Delayed reflection and false positives]
\label{remark:delayed_reflection}
Our analysis extends to settings where the reflection may misidentify a correct transition as incorrect, or identify it with a bounded delay $\Delta_r$; see Appendices \ref{sec:false_positives} and \ref{sec:delayed_reflections}.
\end{remark}

\subsection{Posterior Updates via Conditioning}

Reflection provides information about which reasoning steps are unlikely to lead to correct solutions. For this information to be useful, it must influence future reasoning attempts. In our framework, this is modeled as a progressive in-context downweighting of transitions identified as incorrect.

\begin{definition}[In-context reweighting]
\label{def:in-context_learning}
Let $\mathcal C_t$ denote the in-context history, consisting of previous attempts and reflection events
that identify transitions $(h_r, a_r)$ as invalid.
For a prefix $h$ and candidate next step $a$, define
\begin{equation}
n(h,a;\mathcal C_t)
=
\sum_{(h_r,a_r)\in\mathcal C_t}\mathbf 1\{(h_r,a_r) = (h,a)\},
\end{equation}
the number of times the transition $(h,a)$ has been identified as invalid.
Let $\ell_0(a \mid x, h)$ denote the base logit of generating step $a$ after prefix $h$ given problem $x$. The context-conditioned score is
\begin{equation}
\ell(a \mid x, \mathcal C_t, h)
=
\ell_0(a \mid x, h) - \eta\, n(h,a;\mathcal C_t),
\end{equation}
where $\eta > 0$ is the update rate. Transitions not reflected remain unaffected. 
\end{definition}

Each reflection reduces a transition's relative likelihood by a factor $e^{-\eta}$, so repeated feedback leads to exponential suppression of incorrect prefixes.
In practice, this update is implemented implicitly via conditioning on the in-context history $\mathcal{C}_t$. In Section~\ref{sec:learning}, we show that approximate updates suffice.

\paragraph{Sequential inference.}
We model in-context search as a sequential procedure over rounds $t = 1, 2, \dots$. At each round, a trace $\omega_t$ is sampled from the context-conditioned distribution given the context $\mathcal{C}_{t-1}$. If incorrect, the reflection signal identifies a transition $(h_{t},a_{t})=(y_{<r_t},y_{r_t})$ indexed by $r_t$; both the trace and the identified transition are appended to the context. The process terminates if a correct trajectory is sampled.

This induces a posterior-like update: transitions identified as incorrect are progressively downweighted, concentrating probability on viable reasoning paths.

\paragraph{Reachability.}
At each round, the model selects a prefix from which to continue generation, either starting from the empty prefix or resuming from a prefix present in the context. In all cases, generation is conditioned on the full context.

\begin{assumption}[Contextual reachability]\label{assumption:reachability}
At round $t$, the model selects a restart prefix $h$ equal to either the empty
prefix or a prefix appearing in $\mathcal C_{t-1}$, and generates a continuation
according to the context-conditioned distribution
$P_\theta(\cdot \mid x,\mathcal C_{t-1},h)$.
Resumption is consistent with reflected feedback: if a transition $(h_r,a_r)$ has
been identified as invalid in $\mathcal C_{t-1}$, then no future attempt selects
as its restart prefix any $h$ that extends $(h_r,a_r)$.
\end{assumption}

The contextual reachability assumption rules out pathological behaviors in which the model repeatedly resumes from prefixes already identified as incorrect. Without this restriction, the model could cycle indefinitely.
Appendix~\ref{sec:loops} provides empirical evidence for the necessity of this assumption.

\begin{definition}[In-context search protocol]
\label{def:ics_protocol}
The inference procedure consists of:
\begin{enumerate}
\item sampling trajectories conditioned on the current context,
\item applying adaptive reflection,
\item updating reasoning step logits according to Definition~\ref{def:in-context_learning}.
\end{enumerate}
\end{definition}

\section{When Does In-Context Search Help?}
\label{sec:theory}

\subsection{Failure Modes}

We begin by characterizing regimes in which in-context search does \emph{not} provide an advantage over independent sampling. These results highlight the necessity of informative reflection and clarify the role of early error localization.

\paragraph{Late Error Localization}

A central requirement for efficient in-context search is the ability to identify errors \emph{early} in a reasoning trajectory. Intuitively, the benefit of reflection arises from eliminating large portions of the search space. If reflection only identifies mistakes at late stages, then only a small portion of it is narrowed down.
This intuition is formalized in the following result.

\begin{proposition}\label{res:late_reflections}
(Late reflection is ineffective)
Consider the in-context search protocol of Definition~\ref{def:ics_protocol} with logit updates as in Definition~\ref{def:in-context_learning}. Suppose that:

\begin{enumerate}
\item reflections are restricted to identifying incorrect steps at positions strictly greater than $n/2$ until all such steps have been eliminated;
\item every incorrect prefix of length $d<n/2$ admits at least two possible next steps;
\item $p_1<1$ is the initial probability that the model selects a correct first reasoning step.
\end{enumerate}

Then for any round $t < 2^{n/2}$, the probability of selecting a correct first reasoning step at round $t$ remains equal to its initial value $p_1$. In particular, it remains at most $p_1$ even after $2^{n/2}$ failed attempts.
\end{proposition}

Late reflection fails to eliminate incorrect prefixes at early stages and therefore the space of possible continuations remains large. 
The formal proof is given in Appendix~\ref{proof:late_reflections}.

\paragraph{Misleading Positive Feedback}

One might also consider reinforcing steps believed to be correct via \emph{positive feedback}. However, such feedback is inherently unreliable, as correctness cannot be verified without solving the underlying task.
This leads to the following negative result.

\begin{proposition}\label{res:positive_reflections}
(Positive reflections can harm)
Consider the in-context search update rule of Definition~\ref{def:in-context_learning}, but with the sign reversed so that a reflected transition receives a positive logit update. Let $h$ be a correct prefix and let $a$ be an incorrect next step from $h$ such that the total probability mass assigned to correct next steps from $h$ is positive. If reflection selects the transition $(h,a)$ and increases its logit by $\eta>0$, then the probability of sampling a correct solution strictly decreases.
\end{proposition}

Positive feedback amplifies selected prefixes rather than eliminating regions of the search space; when noisy, it increases the likelihood of extending incorrect reasoning paths and reduces the probability of sampling a correct solution. This highlights a fundamental asymmetry: eliminating incorrect reasoning is substantially more robust than reinforcing correct reasoning. Proof in Appendix~\ref{proof:positive_reflections}.

\subsection{Success Conditions}

We characterize when in-context search outperforms independent sampling. When reflection reliably identifies the \emph{earliest incorrect step}, sequential conditioning transforms an exponentially hard sampling problem into a polynomial one. The key mechanism is that identifying an incorrect prefix eliminates all of its continuations, pruning large regions of the search space rather than ruling out individual trajectories. For simplicity, we state the single-solution case; general case in Appendix~\ref{proof:inference_backtracks}.
\begin{theorem}\label{result:inference_backtracks}
(Efficient in-context search under early reflection)
Consider a reasoning problem with maximum reasoning length $n$ and at most $W$ possible next steps at each prefix. Suppose that:

\begin{enumerate}
\item (Support coverage) The base model satisfies $\gamma$-local support.
\item (Inference protocol) The model follows the in-context search protocol.
\item (Adaptive reflection) The reflection satisfies $\mathbb{E}[Z_t \mid \mathcal{C}_{t-1}] \ge p_r$ whenever the earliest incorrect prefix of \(\omega_t\) has been used at least
\(m_r\) times in \(\mathcal{C}_{t-1}\).
\end{enumerate}

Then for any $\delta\in(0,1)$, with probability at least $1-\delta$, a correct
solution is sampled after at most
\begin{equation}
T
=
\widetilde O\!\left(
nW m_r
+
\frac{nW}{p_r\eta}\log\frac1\gamma 
+
\frac{1}{p_r^2}\log\frac1\delta
\right)
\end{equation}
sampling rounds, where the $\widetilde O(\cdot)$ hides only logarithmic factors
in $n,W,1/\eta,1/p_r,m_r,\log 1/\gamma$ and $\log1/\delta$.
\end{theorem}

The theorem shows that when early error localization is available, the number of required attempts scales polynomially in $n$ and $W$, in contrast to the exponential scaling of independent sampling. This establishes that sequential conditioning can yield an exponential improvement in sampling efficiency.

\paragraph{Extensions and robustness.}
The result extends to more general settings, with corresponding degradations in sample complexity:

\begin{itemize}
\item \textbf{Multiple correct paths.} With $k$ correct solutions, $T = \tilde{O}(k)$ (Appendix \ref{proof:inference_backtracks}).

\item \textbf{Bounded delayed reflection.} If errors are identified within depth $\Delta_r$ of the earliest mistake, $T = \tilde{O}(\exp(\Delta_r))$ (Appendix \ref{sec:delayed_reflections}) - capturing a transition to inefficiency with late feedback.

\item \textbf{False positives.} If each correct transition is incorrectly flagged at most $d$ times, $T = \tilde{O}(d)$ (Appendix \ref{sec:false_positives}). Thus $O(poly(n))$ false positives retains polynomial efficiency.
\end{itemize}

\paragraph{Corollary (Exponential vs.\ polynomial gap).}
In reasoning problems, the base model's pass rate can decay exponentially with depth - if each step is correct with probability at most $p<1$, then the probability of a correct trajectory is at most $p^n = O(\exp(-n))$.
Under Theorem~\ref{result:inference_backtracks}, however, in-context search succeeds with high probability after only $O(n)$ sequential attempts (up to logarithmic factors), yielding an exponential separation between parallel sampling and sequential inference.

\begin{remark}[Local exponential improvement]
The dependence on $\gamma$ is logarithmic, reflecting a local exponential gain. At a single reasoning step, independent sampling requires $\Theta(1/\gamma)$ attempts to select it. In contrast, in-context search suppresses incorrect alternatives across rounds, promoting the correct continuation to high probability after $O(\log(1/\gamma))$ attempts. Thus, the exponential improvement appears already at the level of a single step, and compounds across depth.
\end{remark}

\section{Learning and Emergence}
\label{sec:learning}
Here we study the learning perspective on in-context search. We first show
that the inference-time guarantee of
Theorem~\ref{result:inference_backtracks} is stable under standard token-level training: if a model
accurately imitates the transcript distribution of an efficient search policy,
it inherits its pass-rate guarantee with only polynomial dependence on the
rollout length.
We then give a complementary stagewise analysis, serving as a stylized proxy for
RLVR, showing that under noisy reflection feedback, the posterior update rule of
Definition~\ref{def:in-context_learning} arises as the optimal policy extension.

\subsection{Robustness to Approximate Posterior Updates}\label{res:sft}

The efficiency guarantees of Theorem~\ref{result:inference_backtracks} rely on a specific posterior update rule.
We show that the exact implementation is unnecessary: a step-wise approximation, learnable from supervised log loss on next-step prediction suffices to recover the guarantee.

\begin{proposition}[Step-wise transfer to search pass rate]\label{thm:tokenwise_transfer}
Let $P^*$ denote an autoregressive distribution over search transcripts that satisfies the inference-time guarantee of Theorem~\ref{result:inference_backtracks}. For $\beta\in(0,1)$, let
$T_\beta := T(\beta,n,W)
$
denote the rollout length such that the probability of generating a correct solution by round $T_\beta$ under $P^*$ is at least $1-\beta$.
Let $H$ be a hypothesis class of autoregressive next-step predictors, and let $P_h\in H$ satisfy
\begin{equation}
\mathbb{E}_{x\sim D,\; y\sim P^*(\cdot\mid x)}
\left[
\frac1{|y|}\sum_{i=1}^{|y|}
\log\frac{1}{P_h(y_i\mid x,y_{<i})}
\right]
\le
\frac{\delta\epsilon^2}{4T_{\epsilon/2}} 
\end{equation}
Then for an i.i.d. problem $x\sim D$, with probability at least $1-\delta$, the pass rate of $P_h$ by round $T_{\epsilon/2}$ is at least $1-\epsilon$.
\end{proposition}

The required step-level loss scales as $O(1/T_{\epsilon/2})$. In particular, for fixed reflection parameters, Theorem~\ref{result:inference_backtracks} gives $T_{\epsilon/2}=\widetilde O(n W )$, so standard generalization bounds imply polynomial sample complexity, $\widetilde O(n^2W^2/(\delta^2\epsilon^4))$, up to logarithmic factors. Thus, standard autoregressive training on search rollouts suffices to recover the efficiency of in-context search.

\subsection{Connection to Reinforcement Learning with Verifiable Rewards}\label{res:rlvr}

We provide a principled justification for the posterior reweighting rule in Section~\ref{sec:framework} via an idealized \emph{stagewise policy extension} setting.
This is not a model of the full RLVR training dynamics; rather, it abstracts one empirical feature of RLVR-trained reasoning models: CoT length tends to increase over training \citep{guo2025deepseek}. We interpret this as incremental policy extension, where later training stages learn to append additional reasoning attempts conditioned on earlier attempts and feedback. We ask what update rule is optimal if each extension is chosen to maximize correctness likelihood.

\paragraph{Stagewise policy extension.}
Starting from a base policy that generates a single trajectory, we iteratively extend it by adding additional sampling stages. At stage $t$, the policy generates a trajectory conditioned on previous attempts and reflections. Given a fixed $t$-stage policy, we characterize the optimal $(t+1)$-stage extension for maximizing correctness likelihood.

\paragraph{Noisy reflection model.}
We assume a latent correct reasoning trace $y^\star$ and a sequence of noisy reflections on sampled trajectories. Each reflection identifies a reasoning step along a trajectory and indicates whether it is the first incorrect step (with probability $p_r> \tfrac12$) or the last correct step (with probability $1-p_r$).
Reflections are conditionally independent given $y^\star$, and affect only the likelihood of candidate next steps.
These signals provide local evidence about trajectory correctness and induce a posterior $P_t(y) = \Pr(y^\star = y \mid \mathcal{C}_t)$. Formal definitions given in Appendix~\ref{noisy_reflection_oracle_assumptions}.

\begin{theorem}[Optimal stage extension under noisy reflection]\label{thm:posterior}
Let $\pi^{(t)}$ be a policy generating $t$ trajectories, and let $\mathcal{C}_t$ denote the resulting history. Let $q$ be an autoregressive policy parameterized by per-step logits $\ell_t$. Under the noisy reflection model, the optimal $(t+1)$-stage extension policy in the sense of maximizing
\begin{equation}
\mathbb{E}_{y^\star\sim P_t}[\log q(y^\star \mid \mathcal C_t)]
\end{equation}
satisfies:
\begin{equation}
\ell_t(y_i|y_{<i})=\ell_0(y_i|y_{<i})-\eta n_t(y_{<i},y_i), \qquad
 \eta=\log\left(\frac{p_r}{1-p_r}\right).
\end{equation}
where $n_t(y_{<i},y_i)$ counts the number of times the reflection marked transition $(y_{<i},y_i)$ in $\mathcal{C}_t$.
\end{theorem}

Theorem~\ref{thm:posterior} shows that the exponential reweighting rule of Section~3 arises as the optimal strategy for extending a multi-stage reasoning policy under noisy feedback.
This provides a lens on RLVR: while we do not model the full training dynamics, the result suggests that policies optimized stage-wise under noisy feedback may favor update rules that approximate this form of posterior reweighting.

\section{Empirical Evidence}
\label{sec:experiments}

We provide empirical evidence supporting the qualitative predictions of our theory. Our goal is not a comprehensive empirical evaluation, but to test whether the key mechanisms identified by the theory are reflected in practice. We test two key components:  (i) reflection as an error localization mechanism, and (ii) observable consequences of in-context search in large reasoning models. In appendix \ref{empirical:synthetic} we test our results on synthetic traces with a controlled implementation of the theoretical update and reflection assumptions.

\subsection{Reflection as an Error Localization Mechanism}

Our theory assumes that reflection can identify incorrect reasoning steps, particularly localizing early mistakes with non-negligible probability. 
We evaluate this on structured reasoning traces (arithmetic chains and equation solving) with a single injected error, asking the model to identify the earliest incorrect step. Qwen-3-4B-Instruct was used.

We measure the reflection delay - the gap between the earliest incorrect step and the step identified by the model. Figure~\ref{fig:reflection} shows non-trivial mass at zero delay, together with a sparse tail of delayed reflections.
This supports the early localization assumption, while the tail aligns with the robustness analysis.

\begin{figure}[h!]
    \centering
    \includegraphics[width=0.7\linewidth]{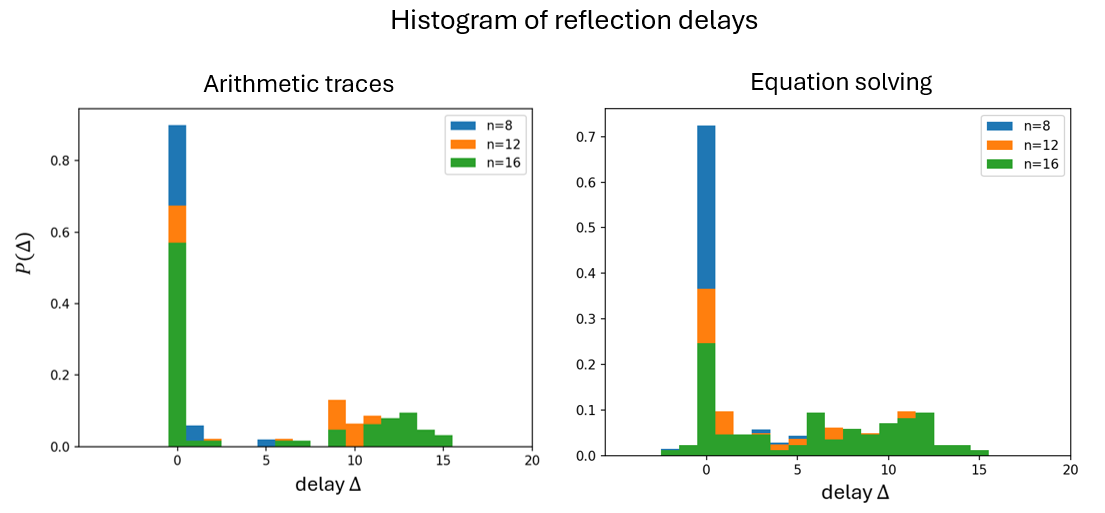}
    \caption{Reflection delay distributions for arithmetic and equation-solving tasks.}
    \label{fig:reflection}
\end{figure}

\subsection{Evidence in Large Reasoning Models}\label{empirical:real}

Theorem~\ref{result:inference_backtracks} suggests that, as reasoning unfolds, probability mass should concentrate on correct solutions. To probe this, we estimate prefix-conditioned pass rates - using them as a proxy for posterior mass - along CoT prefixes. We present results for DeepSeek-R1-Distill-Qwen-2.5-1.5B; complementary results for DeepSeek-R1-Distill-Qwen-2.5-7B are provided in Appendix~\ref{sec:complementary-7B}.

Figure~\ref{fig:spaghetti}(a) shows this quantity for successful trajectories. Pass rates typically increase toward 1 through intermediate gains rather than a single jump, though the evolution is heterogeneous and not strictly monotone. This indicates progressive accumulation of probability mass on correct solutions.
Figure~\ref{fig:spaghetti}(b) shows unsuccessful trajectories. Here, pass rates remain low and highly non-monotonic, suggesting that the dynamics are not simply accumulating progress, but involve redistribution across competing continuations. In appendix \ref{sec:sequential_revisions}, we perform this experiment in a setting more similar to our framework - explicit rounds of trace generation and reflections, yielding qualitatively similar results.

To isolate the effect of self-correction, we segment successful trajectories at correction points and measure log changes in pass rate between segments. Figure~\ref{fig:spaghetti}(c) shows that most updates are small, with both negative values and a long positive tail: some corrections reduce success probability, while a few induce large gains.
This supports a view of self-correction as not simply accumulating progress, but as sparse yet sometimes substantial reallocation of probability mass.

\begin{figure}[h!]
    \centering
    \includegraphics[width=0.9\linewidth]{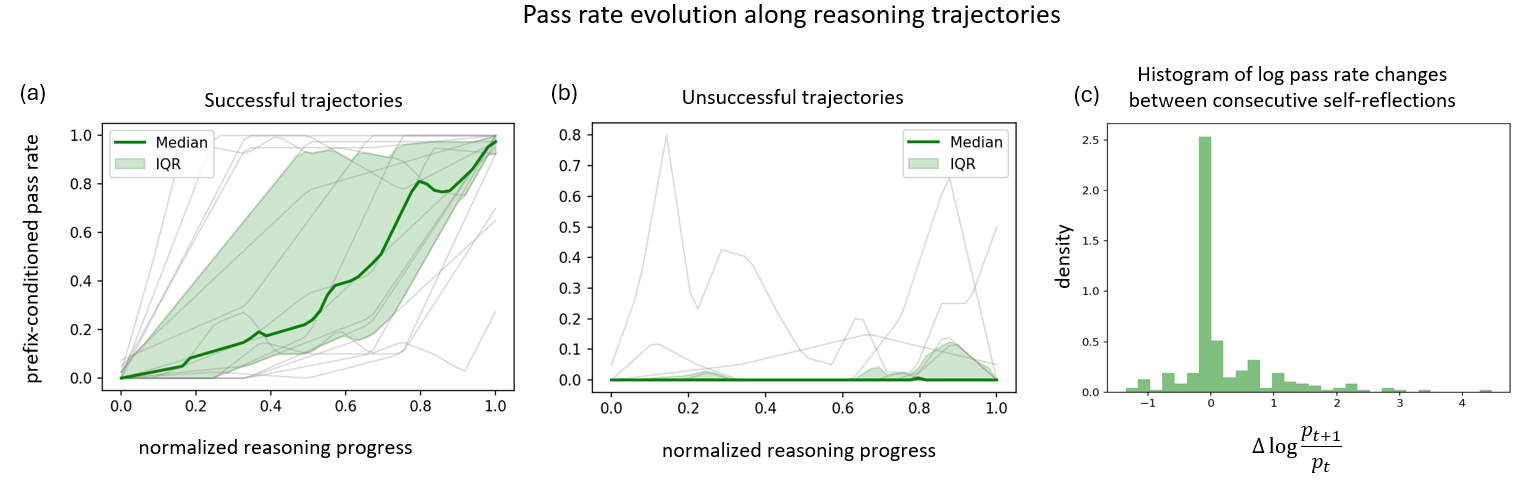}
    \caption{
    Evolution of pass rate conditioned on CoT prefixes along reasoning trajectories on AIME 2025 (DeepSeek-R1-Distill-Qwen-2.5-1.5B), using prefixes sampled at $\sim$1k-token intervals. Gray lines show individual trajectories; green shows the median with interquartile range. (a) Successful trajectories generally increase, with occasional reversals. (b) Unsuccessful trajectories remain low and highly non-monotonic. (c) Reflection-level log changes in pass rate along successful trajectories between consecutive self-correction segments.
    }
    \label{fig:spaghetti}
\end{figure}

\section{Discussion}
\label{sec:discussion}

In this work, we studied in-context search in LRMs as feedback-driven inference over a base model prior. We show that its effectiveness is governed by the quality of reflection: when reflections reliably localize early mistakes, sequential updates can transform an exponentially small pass rate into high success probability using only polynomially many attempts; when reflections are uninformative, conditioning provides no asymptotic improvement over parallel sampling and may even degrade performance. 
We further show that the required update behavior is robust and learnable from supervised rollouts, and that sequential extension under a final correctness reward - mirroring RLVR - the optimal policy extension implements the same update rule, helping explain the emergence of search-like reasoning in practice. We further discuss the limitations of this framework in Appendix \ref{sec:limitations}.

\section{Acknowledgements}

This research was supported by the ERC (European Research Council) and the ISF (Israel Science
Foundation).

\bibliography{main}
\bibliographystyle{plainnat}

%%%%%%%%%%%%%%%%%%%%%%%%%%%%%%%%%%%%%%%%%%%%%%%%%%%%%%%%%%%%

\clearpage
\appendix

\section{Limitations}\label{sec:limitations}

\paragraph{Assumptions:} Our analysis relies on several assumptions that may not hold fully in practice.

\begin{itemize}
    \item First, we assume $\gamma$-local support (assumption \ref{assumption:universality}), ensuring that correct continuations have nonzero probability under the base model. While mild, due to the logarithmic dependence of the results, this condition may fail for difficult problems or poorly calibrated models, in which case no amount of inference can recover correct solutions. Even so, this assumption's limitations has empirical evidence in works such as \citep{yue2025does}, where the RLVR sharpens reasoning traces already existent in the base model, indicating the importance of the base model having sufficient support for correct traces.

    \item Second, the adaptive reflection signal (definition \ref{def:reflection}) requires that the model can identify early incorrect steps with non-negligible probability after repeated exposure.
    In practice, reflection may be more delayed or reinforcing a promising direction instead of pointing to an incorrect one.
    In this work we show these two can harm search, but in practice the model can compensate for this with a smarter update of its distribution over traces.

    \item Third, a central component of our analysis is the local in-context reweighting update (definition \ref{def:in-context_learning}), which multiplicatively downweights invalid prefixes. The purpose of this is so that the model is responsive to the reflection signal. This rule is idealized and may not be implemented exactly by learned models. We show that approximate step-wise updates suffice to preserve the efficiency guarantees. Additionally, the locality can be relaxed further - for the main results of theorem \ref{result:inference_backtracks}, there need not be restrictions on the redistribution of probability mass on steps that succeed the first incorrect step. Even so, the extent to which real models approximate this behavior depends on training dynamics and model capacity.
    \item Fourth, we assume contextual reachability (assumption \ref{assumption:reachability}), ruling out repeated resumption from previously invalidated prefixes. As shown in Appendix~\ref{sec:loops}, current large reasoning models can exhibit reasoning loops in which invalid prefixes are revisited, leading to inefficient search. The results show that this is a major failure mode, responsible for many of the unsuccessful CoTs. Our theoretical results for search efficiency therefore characterize an idealized regime in which such behavior is avoided, as when search fails it is often for this reason.

    \item Reflection parameters: We treat the reflection parameters $p_r$ and $m_r$ as fixed quantities. This dependence may be substantially more favorable than solving the full problem - identifying an incorrect transition in a chain of length $n$ is a much simpler task than searching over an exponentially large space of solutions - but characterizing it empirically and theoretically remains an important direction for future work. We note that if $p_r$ and $m_r$ scale polynomially with $n$, then the efficiency of in-context search is maintained. If they do not, this is an outcome consistent with the failure mode of proposition \ref{res:late_reflections}, of not identifying the earliest incorrect error quickly enough, which leads to no asymptotic improvement over parallel sampling from the base model.
\end{itemize}

\paragraph{Theoretical results:} Our theoretical results focus on worst-case bounds in terms of reasoning length $n$ and branching factor $W$. These quantities may be difficult to estimate in practice, and actual performance may depend on additional factors.
Furthermore, the results related to RLVR as a stage-wise policy extension (Theorem \ref{res:rlvr}), does not explicitly study an RLVR algorithm, but an effective process consistent with the training dynamics of RLVR, where a model gradually increases the length of the search process under a correctness maximization reward \citep{guo2025deepseek}. 

\paragraph{Empirical results:} Our empirical evaluation primarily provides qualitative validation of the theoretical predictions rather than a comprehensive empirical study.

Overall, our results should be interpreted as identifying conditions under which in-context search is efficient, rather than guaranteeing such behavior in all practical settings.

\section{Tree Representation for Analysis}
\label{sec:tree_representation}

\paragraph{Tree representation for analysis.}
For the purpose of analysis, we represent the prefix structure of reasoning traces as a rooted tree. Each node corresponds to a prefix $h = y_{1:i}$, and edges correspond to valid next-step extensions under the base model. A complete trajectory corresponds to a root-to-leaf path.

This representation is purely a notational device: the model itself operates by conditioning on prefixes, as described in Section~\ref{sec:framework}. We use this representation throughout the proofs to reason about how local updates affect sets of continuations.

Each reflection event identifies a transition, which in the tree representation can equivalently be represented by the child node reached by that transition: if the reflection signal returns a transition corresponding to index $r_t$ in a trajectory $(y_1,\dots,y_n)$, this corresponds to identifying the node $v = y_{1:r_t}$, or equivalently the transition from its parent to $v$, as incorrect.

For clarity, we make the correspondence explicit:
\[
\begin{aligned}
\text{prefix } h &\leftrightarrow \text{node } v, \\
\text{next reasoning step } a &\leftrightarrow \text{child of } v, \\\text{transition } (h,a) &\leftrightarrow \text{edge } (v,v'), \\
\text{reasoning trace } (y_1,\dots,y_n) &\leftrightarrow \text{root-to-leaf path}, \\
\text{earliest incorrect prefix } y_{1:k} &\leftrightarrow \text{earliest incorrect node } v_k, \\
\text{reflection outputs transition indexed by } r_t &\leftrightarrow \text{identified  node location on the path}, v_{r_t}\\
Z_t = \mathbf{1}\{r_t = k\}
&\leftrightarrow
Z_t = \mathbf{1}\{v_{r_t} = v_k\}.
\end{aligned}
\]

\section{Proof of Theorem \ref{result:inference_backtracks}}\label{proof:inference_backtracks}

\paragraph{Multiple correct solutions.}
In the main theorem we presented for a single correct reasoning trace, $k=1$, here we present for the general case:

\begin{theorem}
(Efficient in-context search under early reflection with multiple solutions)
Consider a reasoning problem with maximum reasoning length $n$ and at most $W$ possible next steps at each prefix. Suppose that there are $k$ correct reasoning traces and that:

\begin{enumerate}
\item (Support coverage) The base model satisfies $\gamma$-local support (Assumption~\ref{assumption:universality}).

\item (Inference protocol) The model follows the in-context search protocol (Definition~\ref{def:ics_protocol}).

\item (Adaptive reflection) The reflection satisfies that for some
\(p_r\in(0,1]\) and \(m_r\ge 0\), whenever the earliest incorrect prefix in the
sampled trajectory has appeared at least \(m_r\) times in the history, we have
\[
\mathbb{E}[Z_t \mid \mathcal{C}_{t-1}] \ge p_r,
\]
where \(Z_t\) indicates whether the earliest incorrect reasoning step in the
sampled trace is identified (Definition~\ref{def:reflection}).
\end{enumerate}

Then for any $\delta\in(0,1)$, with probability at least $1-\delta$, a correct
solution is sampled after at most
\begin{equation}
T
=
\widetilde O\!\left(
knW m_r
+
\frac{knW}{p_r\eta}\log\frac1\gamma 
+
\frac{1}{p_r^2}\log\frac1\delta
\right)
\end{equation}
sampling rounds, where the $\widetilde O(\cdot)$ hides only logarithmic factors
in $n,k,W,1/\eta,1/p_r,m_r,\log 1/\delta$, and $\log 1/\gamma$.\end{theorem}

As can be seen, the scaling with correct reasoning traces is: $T=\tilde O(k)$.

\textit{proof:}

We analyze the process using the tree representation of prefixes introduced in Section~\ref{sec:tree_representation}. In this representation, prefixes correspond to nodes, and possible next steps correspond to edges. We will refer to prefixes and their extensions interchangeably as nodes and their children.

For clarity, we make the correspondence explicit:
\[
\begin{aligned}
\text{prefix } h &\leftrightarrow \text{node } v, \\
\text{next reasoning step } a &\leftrightarrow \text{child of } v, \\\text{transition } (h,a) &\leftrightarrow \text{edge } (v,v'), \\
\text{reasoning trace } (y_1,\dots,y_n) &\leftrightarrow \text{root-to-leaf path}, \\
\text{earliest incorrect prefix } y_{1:k} &\leftrightarrow \text{earliest incorrect node } v_k, \\
\text{reflection outputs transition indexed by } r_t &\leftrightarrow \text{identified  node location on the path}, v_{r_t}\\
Z_t = \mathbf{1}\{r_t = k\}
&\leftrightarrow
Z_t = \mathbf{1}\{v_{r_t} = v_k\}.
\end{aligned}
\]

Under this correspondence, downweighting a prefix suppresses all continuations passing through the associated node, allowing us to reason about the effect of updates via elimination of subtrees.
We begin the proof with the following lemma:

\begin{lemma}\label{lemma:node_suppression_finite_horizon}
Under $\gamma$-local support, let $u$ be a correct node and let $v$ be an
incorrect child of $u$. Suppose that the reflection signal does not return correct nodes. If $v$ is returned by the reflection signal at least
\[
\frac{1}{\eta}\log\frac{M}{\epsilon\gamma}
\]
times, then the conditional probability of sampling $v$ from prefix $u$ is at
most $\epsilon/M$.
\end{lemma}

\begin{proof}
Fix a correct node $u$ and an incorrect child $v$ of $u$. Let
$q_t(\cdot\mid u)$ denote the model's conditional distribution over children of
$u$ at time $t$. Write this distribution in logit form as
\[
q_t(w\mid u)
=
\frac{q_0(w\mid u)\exp(\Delta_t(w))}
{\sum_{z \in \mathrm{Ch}(u)} q_0(z\mid u)\exp(\Delta_t(z))},
\]
where $\Delta_t(w)$ is the cumulative logit change applied to child $w$ up to
time $t$.

By Definition~\ref{def:in-context_learning}, each time $v$ is identified as
incorrect, its logit is decreased by $\eta$. Hence, after
\[
m \ge \frac{1}{\eta}\log\frac{M}{\epsilon\gamma}
\]
such reflections, we have
\[
\Delta_t(v) \le -\eta m
\le -\log\frac{M}{\epsilon\gamma}.
\]
Therefore,
\[
\exp(\Delta_t(v))
\le
\frac{\epsilon\gamma}{M}.
\]

It follows that
\[
q_t(v\mid u)
=
\frac{q_0(v\mid u)\exp(\Delta_t(v))}
{\sum_{z \in \mathrm{Ch}(u)} q_0(z\mid u)\exp(\Delta_t(z))}
\le
\frac{q_0(v\mid u)\frac{\epsilon\gamma}{M}}
{\sum_{z \in \mathrm{Ch}(u)} q_0(z\mid u)\exp(\Delta_t(z))}.
\]

Now restrict the denominator to correct children of $u$. Let
$C_{\mathrm{corr}}(u)$ denote the set of children of $u$ that lie on at least one
correct root-to-leaf path. Since correct children are not returned by the reflection signal,
%\ref{def:reflection},
by the posterior update rule they are not downweighted. Thus for
every $z\in C_{\mathrm{corr}}(u)$ we have $\Delta_t(z)\ge 0$, and therefore
$\exp(\Delta_t(z))\ge 1$. Thus
\[
\sum_{z \in \mathrm{Ch}(u)} q_0(z\mid u)\exp(\Delta_t(z))
\ge
\sum_{z \in C_{\mathrm{corr}}(u)} q_0(z\mid u).
\]

By $\gamma$-local support, because $u$ is a correct node, the initial
conditional mass assigned to correct children of $u$ satisfies
\[
\sum_{z \in C_{\mathrm{corr}}(u)} q_0(z\mid u) \ge \gamma.
\]
Combining the previous inequalities and using $q_0(v\mid u)\le 1$, we obtain
\[
q_t(v\mid u)
\le
\frac{q_0(v\mid u)\frac{\epsilon\gamma}{M}}{\gamma}
\le
\frac{\epsilon}{M}.
\]
This proves the claim.
\end{proof}

We next show that a node with a sufficiently low likelihood will not be sampled in a given number of turns with high probability:

\begin{lemma}\label{lemma:suppressed_edges_do_not_return}
Fix a deterministic horizon $\overline T$. An incorrect edge
$(u,v)$, where $u$ is a correct node and $v$ is an incorrect child of $u$, is
called suppressed if
\[
q_t(v\mid u)\le \frac{\epsilon}{\overline T knW}.
\]
Then with probability at least $1-\epsilon$, no sampled trajectory traverses a
suppressed edge as its first incorrect edge in any of the first $\overline T$
sampling rounds.
\end{lemma}

\begin{proof}
Let $S$ be the set of correct nodes. Since there are at most $k$ correct
root-to-leaf paths and each has length at most $n$, we have
\[
|S|\le kn.
\]
Each correct node has at most $W$ children, so the number of incorrect children
of correct nodes is at most $knW$.

Fix a round $t\le \overline T$. If the sampled trajectory is incorrect, it has a
first incorrect node. Let $v$ be this first incorrect node and let $u$ be its
parent. Then $u\in S$ and $v$ is an incorrect child of $u$.

For a fixed pair $(u,v)$, the probability that the sampled trajectory traverses
$(u,v)$ as its first incorrect edge is
\[
\Pr_t[\text{reach }u]\cdot q_t(v\mid u).
\]
Where $\Pr_t[\text{reach }u]$ denotes the probability fo reaching $u$ in round $t$. Therefore, the probability that the first incorrect edge is suppressed is at
most
\[
\sum_{u\in S}\sum_{v\in B(u):\, (u,v)\text{ suppressed}}
\Pr_t[\text{reach }u]q_t(v\mid u)
\le
\sum_{u\in S}\sum_{v\in B(u):\, (u,v)\text{ suppressed}}
q_t(v\mid u).
\]
where $B(u)$ denotes the set of incorrect children of $u$. Since each suppressed
edge has conditional probability at most
\[
\frac{\epsilon}{\overline T knW}
\]
and there are at most $knW$ such edges, this probability is at most
\[
knW\cdot \frac{\epsilon}{\overline T knW}
=
\frac{\epsilon}{\overline T}.
\]

Taking a union bound over the first $\overline T$ rounds gives
\[
\Pr\left[
\exists t\le \overline T:
\text{the first incorrect edge sampled at round $t$ is suppressed}
\right]
\le
\overline T\cdot \frac{\epsilon}{\overline T}
=
\epsilon.
\]
Thus, with probability at least $1-\epsilon$, no suppressed edge is sampled as
the first incorrect edge during the first $\overline T$ rounds.
\end{proof}

We now apply to show that after a given number of rounds with first incorrect nodes returned by the reflection, all first incorrect nodes become suppressed and thus only correct solutions are sampled with high probability.

\begin{lemma}\label{lemma:path_suppression_finite_horizon}
Fix a deterministic horizon $\overline T$, and define
\[
L :=
\frac{1}{\eta}\log\frac{\overline T knW}{\epsilon\gamma} .
\]
With probability at least $1-\epsilon$, during the first $\overline T$ rounds,
after at most
\[
N := knW L
\]
incorrect attempts for which the reflection returns the earliest
incorrect node on the sampled trajectory, either a correct solution has already
been sampled, or every future incorrect attempt within the horizon must traverse
a suppressed edge as its first incorrect edge.

Equivalently, on the event from
Lemma~\ref{lemma:suppressed_edges_do_not_return}, after $N$ such successful
earliest-error reflections, no further incorrect trajectory is sampled during
the first $\overline T$ rounds.
\end{lemma}

\begin{proof}
Let $S$ denote the set of correct nodes, i.e. nodes that lie on at least one
correct root-to-leaf path. Since there are at most $k$ correct paths and each has
length at most $n$,
\[
|S|\le kn.
\]

For each correct node $u\in S$, let $B(u)$ denote the set of incorrect children
of $u$. We call an edge $(u,v)$, with $v\in B(u)$, suppressed once
\[
q_t(v\mid u)\le \frac{\epsilon}{\overline T knW}.
\]
Define
\[
L :=
\frac{1}{\eta}\log\frac{\overline T knW}{\epsilon\gamma} .
\]
By Lemma~\ref{lemma:node_suppression_finite_horizon}, for
$M=\overline T knW$, any such edge becomes suppressed after it has been returned
by the reflection signal at least $L$ times.

We work on the high-probability event from
Lemma~\ref{lemma:suppressed_edges_do_not_return}, namely that no sampled
trajectory traverses a suppressed edge as its first incorrect edge during the
first $\overline T$ rounds. This event has probability at least $1-\epsilon$.

Now consider any incorrect sampled trajectory, before a correct solution has
been sampled, for which the reflection signal returns the earliest incorrect
node. Let $v$ be this earliest incorrect node and let $u$ be its parent. Then
$u\in S$ and $v\in B(u)$. We charge this successful earliest-error reflection
to the edge $(u,v)$.

On the high-probability event, every such charge is assigned to an edge that is
not yet suppressed. Indeed, if $(u,v)$ were already suppressed, then the sampled
trajectory would traverse a suppressed edge as its first incorrect edge, which
does not occur on this event.

There are at most $knW$ possible first incorrect edges: each such edge leaves a
correct node, there are at most $kn$ correct nodes, and each correct node has at
most $W$ children.

By Lemma~\ref{lemma:node_suppression_finite_horizon}, any such edge becomes
suppressed after receiving $L$ charges. Hence, on the high-probability event, no
edge can receive more than $L$ charges while continuing to appear as the first
incorrect edge.

Therefore, after at most
\[
knW L
\]
successful earliest-error reflections, every possible first incorrect edge is
suppressed. Indeed, if after $knWL$ such reflections some possible first
incorrect edge had not yet received $L$ charges, then the remaining at most
$knW-1$ possible first incorrect edges would have to account for all the other
charges. Since no edge can receive more than $L$ charges before becoming
suppressed and then disappearing from the charging process, the total number of
charges would be strictly less than $knWL$, a contradiction.

Consequently, after $knWL$ successful earliest-error reflections, any future
incorrect trajectory within the horizon would have to traverse a suppressed edge
as its first incorrect edge. On the high-probability event from
Lemma~\ref{lemma:suppressed_edges_do_not_return}, this cannot happen. Therefore
no further incorrect trajectory is sampled within the horizon unless a correct
solution has already been sampled.

\end{proof}

We now complete the proof of the theorem by accounting for all the rounds where the reflection did not return the first incorrect node.
\begin{proof}
Fix a deterministic horizon $\overline T$, and assume throughout that
$T\le \overline T$. Define
\[
N :=
\frac{k n W }{\eta}
\log\frac{\overline TknW}{\epsilon\gamma}.
\]

Let
\[
\tau_{\mathrm{sol}}
:=
\inf\{t : \text{the trajectory sampled at round }t\text{ is correct}\}
\]
be the first round in which a correct solution is sampled. If
$\tau_{\mathrm{sol}}\le T$, then the theorem is already proved. Hence, in the
remainder of the proof, we analyze the process on the event
$\tau_{\mathrm{sol}}>T$, so that every sampled trajectory up to round $T$ is
incorrect and therefore has a well-defined earliest incorrect node.

By Lemma~\ref{lemma:path_suppression_finite_horizon}, with probability at least
$1-\epsilon$, suppressed edges do not reappear as first incorrect edges during
the first $\overline T$ rounds. On this event, after at most $N$ incorrect
attempts for which the reflection signal returns the earliest incorrect node,
no further incorrect trajectory can be sampled during the first $\overline T$
rounds. Therefore, since $T\le \overline T$, it suffices to show that with
probability at least $1-\delta$, by round $T$ there have been at least $N$
rounds in which the reflection identifies the earliest incorrect node, unless a
correct solution has already been sampled.

\paragraph{Concentration under delayed adaptive reflection.}
We first account for the rounds on which the reflection guarantee
$\mathbb E[Z_t\mid \mathcal C_{t-1}]\ge p_r$ need not yet apply.

Call an incorrect round \emph{eligible} if the earliest incorrect prefix of the
sampled trajectory has already appeared in at least $m_r$ previous failed
trajectories in the history. By the delayed adaptive reflection assumption, on
every eligible incorrect round $t$,
\[
\mathbb E[Z_t\mid \mathcal C_{t-1}] \ge p_r,
\]
where $Z_t$ is the indicator that the reflection signal identifies the earliest
incorrect node.

We claim that there are at most $knWm_r$ ineligible incorrect rounds. Indeed,
group incorrect attempts according to their earliest incorrect prefix. For a
fixed earliest incorrect prefix $v$, at most $m_r$ attempts whose earliest
incorrect prefix is $v$ can be ineligible: after $v$ has appeared in $m_r$
previous failed trajectories, all subsequent failed attempts with earliest
incorrect prefix $v$ are eligible. Moreover, every earliest incorrect prefix is
an incorrect child of a correct node. Since there are at most $kn$ correct nodes
and at most $W$ children per node, the number of possible earliest incorrect
prefixes is at most $knW$. Hence the total number of ineligible incorrect rounds
is at most
\[
T_1 := knWm_r .
\]

It remains to control the number of eligible incorrect rounds needed to obtain
$N$ successful earliest-node reflections. Let
\[
\tau_1,\tau_2,\dots
\]
denote the eligible incorrect rounds, and define
\[
\widetilde Z_j := Z_{\tau_j}.
\]
For every $j$ for which $\tau_j$ is defined, the delayed adaptive reflection
assumption gives
\[
\mathbb E[\widetilde Z_j\mid \mathcal C_{\tau_j-1}] \ge p_r .
\]

Define
\[
D_j
:=
\widetilde Z_j
-
\mathbb E[\widetilde Z_j\mid \mathcal C_{\tau_j-1}],
\qquad
M_R := \sum_{j=1}^R D_j .
\]
and set \(M_0:=0\) and \(\mathcal C_{\tau_0}:=\mathcal C_0\).
By construction:
\[
\mathbb{E}[D_j \mid \mathcal{C}_{\tau_j-1}]
=
\mathbb{E}[\tilde Z_j \mid \mathcal{C}_{\tau_j-1}]
-
\mathbb{E}[\tilde Z_j \mid \mathcal{C}_{\tau_j-1}]
= 0.
\]

Moreover, since
\[
\mathcal C_{\tau_{j-1}}\subseteq \mathcal C_{\tau_j-1},
\]
Where the left hand side is $\tau_{j-1}$ and the right hand side is $\tau_j-1$, the tower property implies
\[
\mathbb E[D_j\mid \mathcal C_{\tau_{j-1}}]
=
\mathbb E\!\left[
\mathbb E[D_j\mid \mathcal C_{\tau_j-1}]
\mid \mathcal C_{\tau_{j-1}}
\right]
=
0.
\]

Also, \(M_j\) is measurable with respect to \(\mathcal C_{\tau_j}\), since it is determined after the first \(j\) eligible incorrect rounds have occurred. Therefore,
\[
\mathbb E[M_j\mid \mathcal C_{\tau_{j-1}}]
=
\mathbb E[M_{j-1}+D_j\mid \mathcal C_{\tau_{j-1}}]
=
M_{j-1}
+
\mathbb E[D_j\mid \mathcal C_{\tau_{j-1}}]
=
M_{j-1}.
\]
Thus \((M_j)\) is a martingale with respect to the filtration
\[
\left(\mathcal C_{\tau_j}\right)_{j\ge 0}.
\]

Since \(\widetilde Z_j\in\{0,1\}\) and
\[
\mathbb E[\widetilde Z_j\mid \mathcal C_{\tau_j-1}]\in[0,1],
\]
we have
\[
|D_j|\le 1
\]
almost surely. Hence
\[
|M_j-M_{j-1}|=|D_j|\le 1.
\]
By Azuma--Hoeffding, for any \(a>0\),

\[
\Pr\!\left[
M_R\le
-a
\right]
\le
\exp\!\left(-\frac{a^2}{2R}\right).
\]

Meaning:
\[
\Pr\!\left[
\sum_{j=1}^R \left(\widetilde Z_j
- 
\mathbb E[\widetilde Z_j\mid \mathcal C_{\tau_j-1}]\right)\le
-a
\right]
\le
\exp\!\left(-\frac{a^2}{2R}\right).
\]

Equivalently:
\[
\Pr\!\left[
\sum_{j=1}^R \widetilde Z_j
\le
\sum_{j=1}^R
\mathbb E[\widetilde Z_j\mid \mathcal C_{\tau_j-1}]
-a
\right]
\le
\exp\!\left(-\frac{a^2}{2R}\right).
\]
Using the definition of the reflection in eligible rounds:
\[
\sum_{j=1}^R
\mathbb E[\widetilde Z_j\mid \mathcal C_{\tau_j-1}]
\ge p_rR
\]
and taking
\[
a=\frac12 p_rR,
\]
we obtain
\[
\Pr\!\left[
\sum_{j=1}^R \widetilde Z_j
\le
\frac12 p_rR
\right]
\le
\exp\!\left(-\frac{p_r^2R}{8}\right).
\]

Therefore, if
\[
R\ge \frac{8}{p_r^2}\log\frac1\delta,
\]
then with probability at least $1-\delta$,
\[
\sum_{j=1}^R \widetilde Z_j \ge \frac12 p_rR.
\]
To guarantee at least $N$ successful earliest-node reflections, it suffices that
\[
\frac12 p_rR \ge N,
\]
or equivalently
\[
R\ge \frac{2N}{p_r}.
\]
Thus accounting for both cases, it is enough to take
\[
R
\ge
\frac{2}{p_r}
\left(
N+\frac{4}{p_r}\log\frac1\delta
\right)
\]
eligible incorrect rounds.

Since there are at most $T_1=knWm_r$ ineligible incorrect rounds, it suffices to
take
\[
T = T_1 + R
\ge
knWm_r
+
\frac{2}{p_r}
\left(
N+\frac{4}{p_r}\log\frac1\delta
\right).
\]
Substituting
\[
N :=
\frac{k n W }{\eta}
\log\frac{\overline TknW}{\epsilon\gamma}
\]
gives the claimed bound.

Combining the concentration event, which fails with probability at most
$\delta$, with the finite-horizon no-return event from
Lemma~\ref{lemma:path_suppression_finite_horizon}, which fails with probability
at most $\epsilon$, we conclude that with probability at least
$1-\delta-\epsilon$, a correct solution is sampled by round $T$.

Substituting $\epsilon,\delta\rightarrow \delta/2$, 
then with probability at least $1-\delta$, a correct solution is sampled by
round $T$, where $T\le \overline T$ and
\[
T
\ge
knWm_r
+
\frac{2}{p_r}
\left(
\frac{k n W }{\eta}
\log\frac{2\overline TknW}{\delta\gamma}
+
\frac{4}{p_r}\log\frac{2}{\delta}
\right).
\]
\[
= knW\left(m_r + \frac{2}{p_r\eta}\log\frac{2\overline TknW}{\delta\gamma}
\right) + \frac{8}{p_r^2}\log\frac{2}{\delta}
\]

\paragraph{From horizon to asymptotic bound.}

It remains to choose the deterministic horizon \(\overline T\) so that the
condition \(T\le \overline T\) used above is self-consistent. From the previous
display, it suffices that
\[
\overline T
\ge
knW\left(
m_r
+
\frac{2}{p_r\eta}
\log\frac{2\overline TknW}{\delta\gamma}
\right)
+
\frac{8}{p_r^2}\log\frac{2}{\delta}.
\]
We use the following elementary claim: for \(A,B,C,x>0\), the inequality
\[
x\ge A+B\log(Cx)
\]
is satisfied whenever
\[
x\ge 2A+2B\log_+(2BC),
\qquad
\log_+(u):=\max\{1,\log u\}.
\]
We prove this claim below in Lemma~\ref{lemma:consistency}.

Apply the claim with
\[
A
:=
knWm_r
+
\frac{8}{p_r^2}\log\frac{2}{\delta},
\qquad
B
:=
\frac{2knW}{p_r\eta},
\qquad
C
:=
\frac{2knW}{\delta\gamma}.
\]
Thus it suffices to take
\[
\overline T
\ge
2knWm_r
+
\frac{16}{p_r^2}\log\frac{2}{\delta}
+
\frac{4knW}{p_r\eta}
\log_+\!\left(
\frac{4knW}{p_r\eta}
\cdot
\frac{2knW}{\delta\gamma}
\right).
\]
For this choice, the condition \(T\le \overline T\) holds, and the previous
finite-horizon argument applies.

Substituting this choice into the logarithmic term gives
\[
\log\frac{2\overline TknW}{\delta\gamma}
=
\log\frac1\gamma
+
O\!\left(
\log(nkW)
+
\log\frac1{p_r}
+
\log\frac1\eta
+
\log m_r
+
\log\log\frac1\delta
+
\log\log\frac1\gamma
\right),
\]
with the convention that logarithms of small terms are absorbed by
\(\log_+\). Therefore, with probability at least \(1-\delta\), a correct
solution is sampled after
\[
T
=
\widetilde O\!\left(
knWm_r
+
\frac{knW}{p_r\eta}\log\frac1\gamma
+
\frac{1}{p_r^2}\log\frac1\delta
\right)
\]
rounds.
\end{proof}

Proof of technical lemma:
\begin{lemma}\label{lemma:consistency}
For \(x,A,B,C>0\), the inequality
\[
x\ge A+B\log(Cx)
\]
is satisfied whenever
\[
x\ge 2A+2B\log_+(2BC),
\qquad
\log_+(u):=\max\{1,\log u\}.
\]
\end{lemma}

\begin{proof}
Using the elementary inequality \(\log y\le y\) for \(y>0\), with
\(y=x/(2B)\), we have
\[
\log\frac{x}{2B}
\le
\frac{x}{2B}.
\]
Equivalently,
\[
\log x
\le
\frac{x}{2B}
+
\log(2B).
\]
Therefore,
\[
A+B\log(Cx)
=
A+B\log C+B\log x
\le
A+\frac{x}{2}+B\log(2BC).
\]
Since
\[
\log(2BC)\le \log_+(2BC),
\]
we further have
\[
A+B\log(Cx)
\le
A+\frac{x}{2}+B\log_+(2BC).
\]
Thus, if
\[
x\ge 2A+2B\log_+(2BC),
\]
then
\[
A+B\log_+(2BC)\le \frac{x}{2},
\]
and hence
\[
A+B\log(Cx)\le x.
\]
This proves the claim.
\end{proof}

\section{Full assumptions for noisy reflection model and proof of Theorem \ref{thm:posterior}}\label{noisy_reflection_oracle_assumptions}

We formalize the assumptions used in Theorem~\ref{thm:posterior}.

\begin{enumerate}
    \item (\textbf{Latent correct trajectory and prior})
    There is a set of candidate trajectories \(\mathcal Y\). A latent variable
    \(y^\star\in\mathcal Y\) denotes the correct trajectory. The initial policy
    \(\pi_0\) induces a prior distribution over trajectories,
    \[
    P_0(y)=\Pr(y^\star=y).
    \]
    Under the autoregressive factorization,
    \[
    P_0(y)=\prod_i P_0(y_i\mid y_{<i}),
    \]
    where \(P_0(\cdot\mid h)\) is parameterized by base logits
    \(\ell_0(\cdot\mid h)\).

    \item (\textbf{Sequential sampling and reflection history})
    At stage \(s=1,\dots,t\), the current policy samples a trajectory
    \[
    y^{(s)}\sim \pi^{(s)}(\cdot\mid \mathcal C_{s-1}).
    \]
    If the sampled trajectory is incorrect, the reflection oracle marks a
    transition
    \[
    (h_s,a_s)=\bigl(y^{(s)}_{<r_s},y^{(s)}_{r_s}\bigr).
    \]
    The reflection history is
    \[
    \mathcal C_s=\{(h_j,a_j)\}_{j=1}^s.
    \]

    \item (\textbf{Noisy reflection})
    The marked transition is truly invalid with probability \(p_r>\tfrac12\).
    Equivalently, if \(A_s\) is the event that \(a_s\) is the correct next step
    after \(h_s\), then
    \[
    \Pr(r_s\mid A_s)=1-p_r,
    \qquad
    \Pr(r_s\mid A_s^c)=p_r.
    \]
    That is, the reflection is more likely to mark \(a_s\) when \(a_s\) is
    invalid than when \(a_s\) is valid.

    \item (\textbf{Prefix locality})
    A reflection on transition \((h_s,a_s)\) affects only the posterior
    next-step distribution after prefix \(h_s\). For every prefix \(h\neq h_s\),
    \[
    P_s(\cdot\mid h)=P_{s-1}(\cdot\mid h).
    \]

    \item (\textbf{No distinction among unmarked alternatives})
    At the marked prefix \(h_s\), the reflection gives no information
    distinguishing between alternatives \(a\neq a_s\). Thus, for any
    \(a,a'\neq a_s\),
    \[
    \frac{P_s(a\mid h_s)}{P_s(a'\mid h_s)}
    =
    \frac{P_{s-1}(a\mid h_s)}{P_{s-1}(a'\mid h_s)}.
    \]

    \item (\textbf{Conditional independence})
    Reflection errors are conditionally independent across stages given the
    latent correct trajectory and sampled trajectories.

    \item (\textbf{Posterior})
    The history \(\mathcal C_s\) induces posterior next-step distributions
    \[
    P_s(a\mid h)
    =
    \Pr(y^\star \text{ takes next step } a \text{ after prefix } h
    \mid \mathcal C_s).
    \]
    These local posteriors induce an autoregressive posterior over trajectories,
    \[
    P_s(y\mid \mathcal C_s)
    =
    \prod_i P_s(y_i\mid y_{<i},\mathcal C_s).
    \]
\end{enumerate}

\begin{proof}
Fix stage \(s\), and let the reflection mark
\[
(h_s,a_s)=\bigl(y^{(s)}_{<r_s},y^{(s)}_{r_s}\bigr).
\]
By prefix locality, for every \(h\neq h_s\),
\[
P_s(\cdot\mid h)=P_{s-1}(\cdot\mid h).
\]

It remains to update the next-step distribution at \(h_s\). Let \(A_s\) denote
the event that \(a_s\) is the correct next step after \(h_s\). The noisy
reflection model gives
\[
\Pr(r_s\mid A_s)=1-p_r,
\qquad
\Pr(r_s\mid A_s^c)=p_r.
\]
Therefore Bayes' rule gives
\[
P_s(a_s\mid h_s)
=
\frac{
(1-p_r)P_{s-1}(a_s\mid h_s)
}{
(1-p_r)P_{s-1}(a_s\mid h_s)
+
p_r\sum_{a\neq a_s}P_{s-1}(a\mid h_s)
}.
\]
For \(a\neq a_s\), the reflection does not distinguish among alternatives, so
their relative probabilities are unchanged. Equivalently,
\[
P_s(a\mid h_s)
\propto
P_{s-1}(a\mid h_s)
\exp(-\eta\mathbf 1\{a=a_s\}),
\qquad
\eta=\log\frac{p_r}{1-p_r}.
\]
Combining this with prefix locality,
\[
P_s(a\mid h)
\propto
P_{s-1}(a\mid h)
\exp(-\eta\mathbf 1\{(h,a)=(h_s,a_s)\}).
\]

Iterating over \(s=1,\dots,t\), and using conditional independence of the
reflection errors across stages, the likelihood contributions multiply:
\[
P_t(a\mid h)
\propto
P_0(a\mid h)
\prod_{s=1}^t
\exp\left(
-\eta \mathbf 1\{(h_s,a_s)=(h,a)\}
\right).
\]

Thus:
\[
P_t(a\mid h)
\propto
P_0(a\mid h)
\exp\left(
-\eta
\sum_{s=1}^t
\mathbf 1\{(h_s,a_s)=(h,a)\}
\right).
\]

By definition of $n_t(h,a)$, we get:
\[
P_t(a\mid h)
\propto
P_0(a\mid h)\exp(-\eta n_t(h,a)).
\]
Thus the corresponding logits satisfy
\[
\ell_t(a\mid h)=\ell_0(a\mid h)-\eta n_t(h,a)
\]
up to an additive constant depending only on \(h\).

The locally updated conditionals induce the trajectory posterior
\[
P_t(y\mid\mathcal C_t)=\prod_i P_t(y_i\mid y_{<i},\mathcal C_t).
\]
For any candidate next-stage policy \(q\),
\[
\mathbb E_{y^\star\sim P_t}[\log q(y^\star\mid\mathcal C_t)]
=
-H(P_t)-\mathrm{KL}(P_t\|q),
\]
which is maximized uniquely at \(q=P_t\). Therefore the optimal stage extension
is the autoregressive policy whose local logits are
\[
\ell_t(a\mid h)=\ell_0(a\mid h)-\eta n_t(h,a).
\]
\end{proof}

\section{Proof of Proposition~\ref{thm:tokenwise_transfer}}\label{proof:PAC}

Fix
\[
T := T(\epsilon/2,n,W),
\]
and let $P^*_{(T)}(\cdot\mid x)$ and $P_{h,(T)}(\cdot\mid x)$ denote the distributions over full search transcripts up to round $T$ induced by $P^*$ and $P_h$, respectively.

By the chain rule for KL divergence,
\begin{align}
\mathbb{E}_{x\sim D}\!\left[
D_{\mathrm{KL}}\!\left(P^*_{(T)}(\cdot\mid x)\,\|\,P_{h,(T)}(\cdot\mid x)\right)
\right]
&=
\mathbb{E}_{x\sim D,\; y\sim P^*_{(T)}(\cdot\mid x)}
\left[
\sum_{i=1}^{|y|}
\log \frac{P^*(y_i\mid x,y_{<i})}{P_h(y_i\mid x,y_{<i})}
\right] \\
&\le
\mathbb{E}_{x\sim D,\; y\sim P^*_{(T)}(\cdot\mid x)}
\left[
\sum_{i=1}^{|y|}
\log \frac{1}{P_h(y_i\mid x,y_{<i})}
\right].
\end{align}
The inequality holds because the entropy term of $P^*$ is nonnegative.

Since every transcript has length at most $T$,
\begin{align}
\mathbb{E}_{x\sim D}\!\left[
D_{\mathrm{KL}}\!\left(P^*_{(T)}(\cdot\mid x)\,\|\,P_{h,(T)}(\cdot\mid x)\right)
\right]
&\le
T\cdot
\mathbb{E}_{x\sim D,\; y\sim P^*_{(T)}(\cdot\mid x)}
\left[
\frac1{|y|}\sum_{i=1}^{|y|}
\log\frac{1}{P_h(y_i\mid x,y_{<i})}
\right] \\
&\le \frac{\delta\epsilon^2}{4}.
\end{align}

Now apply Markov's inequality to the nonnegative random variable
\[
D_{\mathrm{KL}}\!\left(P^*_{(T)}(\cdot\mid x)\,\|\,P_{h,(T)}(\cdot\mid x)\right),
\]
where $x\sim D$. We obtain
\[
\Pr_{x\sim D}\!\left[
D_{\mathrm{KL}}\!\left(P^*_{(T)}(\cdot\mid x)\,\|\,P_{h,(T)}(\cdot\mid x)\right)
>
\frac{\epsilon^2}{4}
\right]
\le \delta.
\]
Therefore, with probability at least $1-\delta$ over $x\sim D$,
\[
D_{\mathrm{KL}}\!\left(P^*_{(T)}(\cdot\mid x)\,\|\,P_{h,(T)}(\cdot\mid x)\right)
\le \frac{\epsilon^2}{4}.
\]

By Pinsker's inequality, on this event,
\[
TV\!\left(P^*_{(T)}(\cdot\mid x),P_{h,(T)}(\cdot\mid x)\right)
\le
\sqrt{\frac{1}{2}D_{\mathrm{KL}}\!\left(P^*_{(T)}(\cdot\mid x)\,\|\,P_{h,(T)}(\cdot\mid x)\right)}
\le \frac{\epsilon}{2\sqrt{2}}
\le \frac{\epsilon}{2}.
\]

Let $A_x$ denote the event that a correct solution is generated by round $T$ on problem $x$. By Theorem~\ref{result:inference_backtracks} with failure parameter $\epsilon/2$,
\[
P^*_{(T)}(A_x\mid x)\ge 1-\frac{\epsilon}{2}.
\]
Moreover, for any event $A$ and distributions $P,Q$,
\[
Q(A)\ge P(A)-TV(P,Q).
\]
Applying this with $A=A_x$, $P=P^*_{(T)}(\cdot\mid x)$, and $Q=P_{h,(T)}(\cdot\mid x)$ yields
\[
P_{h,(T)}(A_x\mid x)
\ge
P^*_{(T)}(A_x\mid x)
-
TV\!\left(P^*_{(T)}(\cdot\mid x),P_{h,(T)}(\cdot\mid x)\right)
\ge
1-\frac{\epsilon}{2}-\frac{\epsilon}{2}
=
1-\epsilon.
\]

Hence, with probability at least $1-\delta$ over $x\sim D$, the pass rate of $P_h$ by round $T(\epsilon/2,n,W)$ is at least $1-\epsilon$.
\qed

\section{Proof of Proposition \ref{res:late_reflections}}\label{proof:late_reflections}

\begin{proof}
We analyze the process using the tree representation of prefixes introduced in Section~\ref{sec:tree_representation}. In this representation, prefixes correspond to nodes, and possible next steps correspond to edges. We will refer to prefixes and their extensions interchangeably as nodes and their children.

Under Definition~\ref{def:in-context_learning}, a reflection on a node $v$ decreases only the logit of $v$ itself; it does not directly alter the logits of any ancestor of $v$. In particular, the probability of selecting a first-step node can change only if that first-step node is itself returned by the reflection signal.

By assumption, until all incorrect nodes at depths greater than $n/2$ have been eliminated, the reflection signal returns only such late nodes. Hence during this phase no node at depth $1$ is ever penalized, and therefore the distribution over first reasoning steps remains unchanged from its initial value. In particular, the probability of selecting a correct first step remains $p_1$ throughout this phase.

It remains to lower bound the duration of this phase. Consider any subtree rooted at an incorrect first-step node. By assumption, every incorrect node at depth $d<n/2$ has at least two children. Therefore, after descending for $n/2$ additional levels, this subtree contains at least $2^{n/2}$ descendants at depth greater than $n/2$ (up to an inconsequential off-by-one depending on indexing). Each failed attempt can eliminate at most one such late incorrect node, since each round produces at most one reflected node. Hence before $2^{n/2}$ failed attempts, at least one late incorrect node remains in each such subtree.

Therefore the reflection process may continue to return only nodes at depth greater than $n/2$ throughout the first $2^{n/2}$ failed attempts, and during all of these rounds the first-step distribution is unchanged. Consequently, the probability of selecting a correct first step remains $p_1$, and in particular is at most $p_1$, even after $2^{n/2}$ failed attempts.
\end{proof}
%%%%%%%%%%%%%%%%%%%%%%%%%%%%%%%%%%%%%%%%%%%%%%%%%%%%%%%%%%%%

\section{Proof of Proposition \ref{res:positive_reflections}}\label{proof:positive_reflections}

\begin{proof}
We analyze the process using the tree representation of prefixes introduced in Section~\ref{sec:tree_representation}. In this representation, prefixes correspond to nodes, and possible next steps correspond to edges. We will refer to prefixes and their extensions interchangeably as nodes and their children.

Let $u$ be the parent of the reflected node $v$, and suppose $u$ is a correct prefix while $v$ is an incorrect child of $u$.

Let $q_t(\cdot \mid u)$ denote the model's conditional distribution at round $t$ over the children of $u$. Under the positive-feedback update, the logit of $v$ is increased by $\eta>0$, while the logits of all other children of $u$ remain unchanged. Therefore
\[
q_{t+1}(v\mid u)
=
\frac{q_t(v\mid u)e^\eta}
{q_t(v\mid u)e^\eta + \sum_{w\neq v} q_t(w\mid u)}
>
q_t(v\mid u).
\]
Equivalently, the total conditional mass assigned to all children other than $v$ strictly decreases:
\[
\sum_{w\neq v} q_{t+1}(w\mid u)
<
\sum_{w\neq v} q_t(w\mid u).
\]

Now let $C(u)$ denote the set of correct children of $u$. Since $v$ is incorrect, we have $v\notin C(u)$. Moreover, by assumption,
\[
\sum_{w\in C(u)} q_t(w\mid u) > 0.
\]
Because only the relative normalization among children of $u$ changes, every child $w\neq v$ has its conditional probability multiplied by the same factor
\[
\frac{1}{q_t(v\mid u)e^\eta + \sum_{w\neq v} q_t(w\mid u)} < 1
\]
relative to its previous mass outside the boosted term. Hence
\[
\sum_{w\in C(u)} q_{t+1}(w\mid u)
<
\sum_{w\in C(u)} q_t(w\mid u).
\]

All branching probabilities above $u$ are unchanged, and all conditional distributions below each child of $u$ are unchanged. Therefore the total probability of sampling a correct complete path through $u$ strictly decreases, while the probability of correct complete paths not passing through $u$ remains unchanged. It follows that the overall probability of sampling a correct solution strictly decreases:
\[
\Pr_{y\sim P_{t+1}(\cdot\mid x)}[y \text{ is correct}]
<
\Pr_{y\sim P_t(\cdot\mid x)}[y \text{ is correct}].
\]
\end{proof}

\section{False Positives in Reflection}\label{sec:false_positives}

We extend the reflection model to allow bounded false positives, where correct prefixes may be incorrectly identified for a bounded number of times.

\begin{assumption}[Bounded false-positive reflection]
\label{assumption:bounded_false_positive_reflection}
Let \(\mathcal{C}_{t-1}\) denote the history up to round \(t-1\). At round \(t\),
suppose the sampled reasoning trace \(\omega_t=(y_1,\dots,y_n)\) is incorrect. Let
\(k\) denote the first invalid transition along the trace: the smallest
index such that the prefix \(y_{1:k}\) does not lie on any correct reasoning
trace.

A reflection signal identifies a transition
$
(y_{<r_t}, y_{r_t})
$
as invalid. Define
\[
Z_t=\mathbf{1}\{r_t=k\}.
\]
We assume that for some \(p_r\in(0,1]\) and \(m_r \ge 0\), if the earliest incorrect prefix
\(y_{1:k}\) was used in at least \(m_r\) previous failed attempts in
\(\mathcal{C}_{t-1}\), then
\[
\mathbb{E}[Z_t\mid \mathcal{C}_{t-1}]\ge p_r.
\]
If $r_t\neq k$, it may return either an incorrect transition after $k$ or a transition before $k$. In addition, for every correct transition $s$, the total number of times $s$ is returned by the reflection signal is at most $d$:
\[
\sum_{t\ge 1}\mathbf{1}\{r_t=s\}\le d.
\]
\end{assumption}

Under this modified reflection signal, in-context search remains efficient, with a degradation that scales with the false-positive budget:

\begin{theorem}[Efficient in-context search under bounded false positives]
Under the assumptions of Theorem~\ref{result:inference_backtracks}, except that the reflection signal satisfies Assumption~\ref{assumption:bounded_false_positive_reflection}, Then for any $\delta\in(0,1)$, if
\[
T
=
\widetilde O\!\left(
knWm_r + 
\frac{knW}{p_r\eta}\left(\log\frac1\gamma + d\right)
+
\frac{1}{p_r^2}\log\frac1\delta
\right)
\]
then with probability at least $1-\delta$, a correct solution is generated by round $T$.
\end{theorem}

As can be seen, a dependence of $\tilde O(d)$ is added, coupled to the depth $n$, width $W$, and number of correct solutions $k$. Letting $d$ scale polynomially with $n$ retains the polynomial efficiency.

\textit{proof:}

We analyze the process using the tree representation of prefixes introduced in Section~\ref{sec:tree_representation}, as was done in the proof of the main theorem (Appendix \ref{proof:inference_backtracks}). In this representation, prefixes correspond to nodes, and possible next steps correspond to edges. We will refer to prefixes and their extensions interchangeably as nodes and their children.

The key ingredient is a robust version of the suppression argument that accounts for possible penalties on correct nodes:

\begin{lemma}\label{lemma:node_suppression_finite_horizon_false_positives}
Suppose the reflection signal satisfies
Assumption~\ref{assumption:bounded_false_positive_reflection}. Under
$\gamma$-local support, let $u$ be a correct node and let $v$ be an incorrect
child of $u$. Fix $M\ge 1$ and $\epsilon\in(0,1)$. If $v$ is returned by the
reflection signal at least
\[
L_d
:=
\frac{1}{\eta}\log\frac{M}{\epsilon\gamma}
+d
\]
times, then the conditional probability of sampling $v$ from prefix $u$ is at
most $\epsilon/M$.
\end{lemma}

\begin{proof}
Fix a correct node $u$ and an incorrect child $v$ of $u$. Let
$q_t(\cdot\mid u)$ denote the model's conditional distribution over children of
$u$ at time $t$. As before, write
\[
q_t(w\mid u)
=
\frac{q_0(w\mid u)\exp(\Delta_t(w))}
{\sum_{z\in \mathrm{Ch}(u)}q_0(z\mid u)\exp(\Delta_t(z))},
\]
where $\Delta_t(w)$ is the cumulative logit change applied to child $w$ up to
time $t$.

Each time $v$ is returned by the reflection signal, its logit is decreased by
$\eta$. Thus, after $m$ such returns,
\[
\Delta_t(v)\le -\eta m.
\]
If
\[
m
\ge
\frac{1}{\eta}\log\frac{M}{\epsilon\gamma}
+d,
\]
then
\[
\exp(\Delta_t(v))
\le
\exp(-\eta m)
\le
\frac{\epsilon\gamma}{M}e^{-\eta d}.
\]

Now let
\[
C_{\mathrm{corr}}(u)
:=
\{z\in \mathrm{Ch}(u): z \text{ lies on at least one correct root-to-leaf path}\}.
\]
By $\gamma$-local support,
\[
\sum_{z\in C_{\mathrm{corr}}(u)}q_0(z\mid u)\ge \gamma.
\]
Under Assumption~\ref{assumption:bounded_false_positive_reflection}, every
correct transition is returned by the reflection signal at most $d$ times. Hence every
correct child $z\in C_{\mathrm{corr}}(u)$ is downweighted by at most $\eta d$,
so
\[
\Delta_t(z)\ge -\eta d
\qquad\text{and hence}\qquad
\exp(\Delta_t(z))\ge e^{-\eta d}.
\]
Therefore the denominator satisfies
\[
\sum_{z\in \mathrm{Ch}(u)}q_0(z\mid u)\exp(\Delta_t(z))
\ge
\sum_{z\in C_{\mathrm{corr}}(u)}q_0(z\mid u)\exp(\Delta_t(z))
\ge
\gamma e^{-\eta d}.
\]

Combining the numerator and denominator bounds gives
\[
q_t(v\mid u)
\le
\frac{q_0(v\mid u)\frac{\epsilon\gamma}{M}e^{-\eta d}}
{\gamma e^{-\eta d}}
\le
\frac{\epsilon}{M},
\]
since $q_0(v\mid u)\le 1$. This proves the claim.
\end{proof}

With this modification in place, the remainder of the proof follows the proof of
Theorem~\ref{result:inference_backtracks}, with adjusted constants.

\begin{proof}
The proof follows the proof of Theorem~\ref{result:inference_backtracks}, with
the robust node-suppression lemma replacing
Lemma~\ref{lemma:node_suppression_finite_horizon}. The only change is that
bounded false-positive reflections may decrease the logits of correct children.
Since every correct transition is returned by the reflection signal at most \(d\) times,
the logit of every correct child can be decreased by at most \(\eta d\).
Therefore the effective lower bound on correct-continuation mass from any
correct prefix is reduced from \(\gamma\) to
\[
\gamma e^{-\eta d}.
\]

Consequently, every occurrence of
\[
\frac{1}{\eta}\log\frac{M}{\epsilon\gamma}
\]
in the node-suppression argument is replaced by
\[
\frac{1}{\eta}\log\frac{M}{\epsilon\gamma e^{-\eta d}}
=
\frac{1}{\eta}\log\frac{M}{\epsilon\gamma}+d.
\]
In the finite-horizon argument we take \(M=\overline TknW\), so the suppression
threshold becomes
\[
L_d
:=
\frac{1}{\eta}\log\frac{\overline TknW}{\epsilon\gamma}
+d .
\]
Thus the path-suppression lemma gives
\[
N_d := knW L_d
\]
successful earliest-error reflections as a sufficient number to suppress every
possible first incorrect edge within the horizon.

The no-return lemma is unchanged: once an incorrect edge satisfies
\[
q_t(v\mid u)\le \frac{\epsilon}{\overline TknW},
\]
with probability at least \(1-\epsilon\), no sampled trajectory traverses such a
suppressed edge as its first incorrect edge during the first \(\overline T\)
rounds.

The martingale concentration argument is also unchanged, since the event
\[
Z_t=\mathbf{1}\{r_t \text{ is the earliest incorrect node}\}
\]
still satisfies, on every eligible incorrect round,
\[
\mathbb{E}[Z_t\mid \mathcal{C}_{t-1}]\ge p_r.
\]
As in the proof of Theorem~\ref{result:inference_backtracks}, there are at most
\[
T_1 := knWm_r
\]
ineligible incorrect rounds. Hence, with probability at least \(1-\delta\), at
least \(N_d\) successful earliest-error reflections occur by round \(T\) whenever
\[
T
\ge
knWm_r
+
\frac{2}{p_r}
\left(
N_d
+
\frac{4}{p_r}\log\frac{1}{\delta}
\right).
\]

After \(N_d\) successful earliest-error reflections, every possible first
incorrect edge is suppressed. Therefore any later incorrect trajectory within
the horizon would have to traverse a suppressed edge as its first incorrect
edge, which does not occur on the high-probability no-return event. Hence a
correct solution is generated by round \(T\), except with the combined failure
probability from the concentration event and the no-return event.
\end{proof}

\section{Delayed Reflections}\label{sec:delayed_reflections}

We now weaken the reflection model by allowing the reflection to identify errors only up to a bounded delay, while ensuring that reflection signals eventually propagate upward.

For the analysis of this extension, we work with the tree representation of prefixes introduced in Appendix~\ref{sec:tree_representation}, as was done in the proof of the main theorem (Appendix \ref{proof:inference_backtracks}). This representation makes it possible to reason about how reflection feedback propagates across multiple levels of continuation. In this view, nodes correspond to prefixes, edges to transitions and subtrees correspond to sets of continuations.

\begin{assumption}[Bounded useful delay with upward propagation]
\label{assumption:bounded_delay_reflection}
At round \(t\), suppose the sampled trace
\[
\omega_t=(v_0,\dots,v_n)
\]
is incorrect, and let \(v_k\) denote its earliest incorrect node. The reflection
signal returns an invalid node \(r_t\) on the sampled trace.

There exist \(p_r\in(0,1]\), \(m_r\ge 0\), and \(\Delta_r\ge 0\) such that, if
the earliest incorrect node \(v_k\) has appeared in at least \(m_r\) previous
failed attempts in \(\mathcal C_{t-1}\), then
\[
\Pr\!\left[
r_t\in\{v_k,v_{k+1},\dots,v_{\min(k+\Delta_r,n)}\}
\mid \mathcal C_{t-1}
\right]
\ge p_r.
\]
We call such a reflection a \(\Delta_r\)-useful reflection.

Moreover, the reflection signal satisfies upward propagation: if a trace
passes through an incorrect node \(v\), and all children of \(v\) that can be
returned as useful delayed reflections have already been eliminated from the
reflection process, then any future useful delayed reflection on such a
trace returns \(v\) rather than one of those children.
\end{assumption}
To control the behavior of delayed reflections within incorrect subtrees, we assume that all local transitions have non-negligible prior probability:

\begin{assumption}[Uniform local support]
For every edge $(v,u)$ in the reasoning tree,
\[
P_\theta(u \mid v,x) \ge \gamma .
\]
\end{assumption}

Under these assumptions, in-context search remains efficient, with a degradation that depends on the delay parameter:

\begin{theorem}
(Efficient in-context search under bounded-delay reflection)
\label{result:bounded_delay_reflections}
Consider a reasoning problem with maximum reasoning length $n$, branching
factor at most $W$, and at most $k$ correct reasoning traces. Suppose the base
model satisfies $\gamma$-local support and uniform local support with parameter
$\gamma$, the model follows the in-context search protocol, and the reflection
signal satisfies bounded-delay reflection with upward propagation.

Define
\[
B_{\Delta_r}:=\sum_{j=0}^{\Delta_r}W^j .
\]
Then for any $\delta\in(0,1)$, with probability at least $1-\delta$, a correct
solution is sampled after at most
\[
T
=
\widetilde O\!\left(
knWm_r
+
\frac{knW B_{\Delta_r}}{p_r\eta}\log\frac1\gamma
+
\frac{1}{p_r^2}\log\frac1\delta
\right)
\]
rounds.
\end{theorem}

Notably, the dependence on the delay parameter $\Delta_r$ is exponential through the factor $B_{\Delta_r}$, reflecting the need to eliminate an entire subtree of depth $\Delta_r$ before error signals propagate upward. This yields a smooth transition from the efficient regime of early error localization to the exponentially inefficient regime where reflections occur only at late stages.

\textit{proof:}

We analyze the process using the tree representation of prefixes introduced in Section~\ref{sec:tree_representation}. In this representation, prefixes correspond to nodes, and possible next steps correspond to edges. We will refer to prefixes and their extensions interchangeably as nodes and their children.

We begin by showing that incorrect nodes within the delayed reflection region can be eliminated from the reflection signal after a sufficient number of reflections.

\begin{lemma}[One-step upward propagation after child discovery]
\label{lemma:one_step_upward_propagation}
Fix a deterministic horizon $\overline T$. Let $v$ be an incorrect node, and
suppose that the reflection signal may return children of $v$ on sampled
trajectories passing through $v$. Define
\[
L :=
\frac{1}{\eta}
\log\frac{\overline T\, M}{\epsilon\gamma},
\]
where $M$ is an upper bound on the number of delayed-region edges considered in
the proof.

Call a child $z$ of $v$ suppressed once
\[
q_t(z\mid v)\le \frac{\epsilon}{\overline T M}.
\]
Assume the no-return event that no sampled trajectory traverses a suppressed
edge in the delayed region during the first $\overline T$ rounds.

Then, after at most $WL$ useful delayed reflections that return children of
$v$, every child of $v$ that can appear on sampled trajectories has been
returned by the reflection signal at least once. Consequently, by upward
propagation, any future useful delayed reflection on a trajectory passing through
$v$ returns $v$ rather than one of its children.
\end{lemma}

\begin{proof}
Fix an incorrect node $v$. Since the branching factor is at most $W$, the node
$v$ has at most $W$ children.

If all children of $v$ that can appear on sampled trajectories have already been
returned by the reflection signal at least once, then the first conclusion
holds, and the upward-propagation assumption implies the second conclusion.
Thus suppose that at least one such child has not yet been returned. Let
$z^\star$ be a child of $v$ that has not yet been returned by the reflection
signal.

We first show that, as long as such an unreturned child $z^\star$ exists, any
other child $z$ that receives $L$ returns becomes suppressed. Since $z^\star$
has not been returned by the reflection signal, the update rule has not
downweighted the transition from $v$ to $z^\star$. Hence its cumulative logit
change satisfies
\[
\Delta_t(z^\star)\ge 0.
\]
By uniform local support,
\[
q_0(z^\star\mid v)\ge \gamma.
\]
Therefore, in the logit representation,
\[
\sum_{y\in \mathrm{Ch}(v)}q_0(y\mid v)\exp(\Delta_t(y))
\ge
q_0(z^\star\mid v)\exp(\Delta_t(z^\star))
\ge
\gamma.
\]

Now let $z$ be any child of $v$ that has been returned at least $L$ times. Each
time $z$ is returned as incorrect, the update rule decreases its logit by
$\eta$. Hence
\[
\Delta_t(z)\le -\eta L
\le
-\log\frac{\overline T M}{\epsilon\gamma},
\]
and so
\[
\exp(\Delta_t(z))
\le
\frac{\epsilon\gamma}{\overline T M}.
\]
Thus
\[
q_t(z\mid v)
=
\frac{q_0(z\mid v)\exp(\Delta_t(z))}
{\sum_{y\in \mathrm{Ch}(v)}q_0(y\mid v)\exp(\Delta_t(y))}
\le
\frac{q_0(z\mid v)\frac{\epsilon\gamma}{\overline T M}}
{\gamma}
\le
\frac{\epsilon}{\overline T M}.
\]
Therefore $z$ is suppressed.

We now count child-returning useful delayed reflections. Suppose, for
contradiction, that after $WL$ such reflections there is still some child
$z^\star$ of $v$ that has not been returned by the reflection signal. Then all
$WL$ reflections were assigned to children other than $z^\star$. Since there
are at most $W-1$ such children, by the pigeonhole principle some returned
child receives at least $L$ returns and hence becomes suppressed by the argument
above.

More generally, while $z^\star$ remains unreturned, every child other than
$z^\star$ can receive at most $L$ returns before becoming suppressed. On the
no-return event, once a child is suppressed, sampled trajectories do not
traverse the corresponding edge again within the horizon, so that child cannot
continue to be returned by useful delayed reflections. Hence the children other
than $z^\star$ can absorb at most $(W-1)L<WL$ such reflections while
$z^\star$ remains unreturned, contradicting the assumption that $WL$ such
reflections occurred.

Therefore, after at most $WL$ useful delayed reflections that return children
of $v$, every child of $v$ that can appear on sampled trajectories has been
returned by the reflection signal at least once. The upward-propagation
assumption then implies that future useful delayed reflections on trajectories
passing through $v$ return $v$ itself rather than one of its children.
\end{proof}

We next use a recursive argument to exhaust the nodes that can be returned by the reflection to the earliest incorrect steps:

\begin{lemma}[Delayed subtree propagation]
\label{lemma:delayed_subtree_propagation}
Fix a deterministic horizon $\overline T$. For $d\ge 0$, define
\[
B_d := \sum_{j=0}^d W^j .
\]
Let $v$ be an incorrect node, and consider useful delayed reflections on
sampled trajectories passing through $v$ whose returned node lies in the
depth-$d$ subtree rooted at $v$. Let
\[
L :=
\frac{1}{\eta}
\log\frac{\overline T\,M}{\epsilon\gamma} .
\]
Work on the no-return event for suppressed edges in the delayed region during
the first $\overline T$ rounds.

Then, after at most \(2B_dL\) such useful delayed reflections,  every future useful delayed
reflection on a trajectory passing through \(v\) returns \(v\) itself.
\end{lemma}

\begin{proof}
We prove the claim by induction on \(d\), working throughout on the no-return
event for suppressed edges in the delayed region during the first
\(\overline T\) rounds.

\paragraph{Base case: \(d=0\).}
If \(d=0\), the depth-\(0\) subtree rooted at \(v\) contains only \(v\) itself.
Thus every useful delayed reflection whose returned node lies in this subtree
already returns \(v\). Hence the claim holds.

\paragraph{Inductive step.}
Assume the claim holds for depth \(d-1\), and consider an incorrect node \(v\)
with delayed region of depth \(d\).

Let \(z\) be a child of \(v\). The portion of the depth-\(d\) subtree rooted at
\(v\) that lies below \(z\) is contained in the depth-\((d-1)\) subtree rooted
at \(z\). Therefore, by the induction hypothesis, after at most
\[
B_{d-1}L
\]
useful delayed reflections whose returned node lies in the depth-\((d-1)\)
subtree rooted at \(z\), either the edge \((v,z)\) entering \(z\) has already
been suppressed, or future useful delayed reflections on trajectories passing
through \(z\) return \(z\) itself.

On the no-return event, if the edge \((v,z)\) is suppressed, then sampled
trajectories do not traverse this edge again within the horizon. Thus such a
child \(z\) can no longer absorb useful delayed reflections below \(v\).

Consequently, after at most
\[
W B_{d-1}L
\]
useful delayed reflections whose returned node lies strictly below a child of
\(v\), all children \(z\) of \(v\) have been reduced to the following state:
either the edge \((v,z)\) is suppressed, or future useful delayed reflections
on trajectories passing through \(z\) return \(z\) itself.

Now consider the remaining useful delayed reflections through \(v\) that return
children of \(v\) themselves. By Lemma~\ref{lemma:one_step_upward_propagation},
after at most
\[
WL
\]
such child-returning useful delayed reflections, either all children of \(v\)
that can appear on sampled trajectories have been returned at least once, in
which case upward propagation forces future useful delayed reflections through
\(v\) to return \(v\), or a child edge has become suppressed and, on the
no-return event, can no longer be traversed. In either case, after these
additional \(WL\) reflections, no child of \(v\) can continue to absorb useful
delayed reflections below \(v\).

Therefore, after at most
\[
W B_{d-1}L + WL
=
W(B_{d-1}+1)L
\]
useful delayed reflections in the depth-\(d\) subtree rooted at \(v\), useful
delayed reflections can no longer be absorbed by strict descendants of \(v\).
Since \(B_d=1+WB_{d-1}\) and \(W\le B_d\), we have
\[
W(B_{d-1}+1)L
=
(WB_{d-1}+W)L
\le
2B_dL.
\]

At this point, no strict descendant of \(v\) can continue to absorb useful
delayed reflections within the depth-\(d\) delayed region. Therefore, by the
upward-propagation property, every future useful delayed reflection on a
trajectory passing through \(v\) returns \(v\) itself rather than one of its
strict descendants. This proves the induction step.

This proves the induction step and hence the lemma.
\end{proof}

Now we can use the above lemma to reduce to the original reflection definition \ref{def:reflection}:

\begin{lemma}[Reduction to early reflection after delayed propagation]
\label{lemma:bounded_delay_reduces_to_early_reflection}
Fix a deterministic horizon $\overline T$, and define
\[
B_{\Delta_r}:=\sum_{j=0}^{\Delta_r}W^j,
\qquad
L :=
\frac{1}{\eta}
\log\frac{\overline T\,M}{\epsilon\gamma} .
\]
Work on the no-return event for suppressed edges in the delayed region during
the first $\overline T$ rounds.

Let \(v\) be an earliest incorrect node on a sampled trajectory. After at most
\[
2B_{\Delta_r}L
\]
useful delayed reflections on trajectories whose earliest incorrect node is
\(v\), every future useful delayed reflection on such a trajectory returns
\(v\) itself. Consequently, from that point onward, useful delayed reflections
satisfy the early-reflection condition of Definition~\ref{def:reflection} for
the earliest incorrect node \(v\).
\end{lemma}

\begin{proof}
Apply Lemma~\ref{lemma:delayed_subtree_propagation} with
\(d=\Delta_r\) to the subtree rooted at \(v\). By definition, a useful delayed
reflection on a trajectory whose earliest incorrect node is \(v\) returns a node
in the depth-\(\Delta_r\) subtree rooted at \(v\). Hence, after at most
\(2B_{\Delta_r}L\) such useful delayed reflections, every future useful delayed
reflection on a trajectory passing through \(v\) returns \(v\) itself rather
than a strict descendant of \(v\).

Since \(v\) is the earliest incorrect node of the trajectory, this is exactly
the early-reflection event from Definition~\ref{def:reflection}. This proves
the claim.
\end{proof}

We now complete the proof of
Theorem~\ref{result:bounded_delay_reflections}. Fix a deterministic horizon
\(\overline T\), and define
\[
B_{\Delta_r}:=\sum_{j=0}^{\Delta_r}W^j.
\]
Let
\[
M := knW B_{\Delta_r}
\]
be an upper bound on the number of delayed-region edges attached to possible
first incorrect edges, and define
\[
L :=
\frac{1}{\eta}
\log\frac{\overline T\,knW B_{\Delta_r}}{\epsilon\gamma} .
\]

On the no-return event for suppressed delayed-region edges, which holds with
probability at least \(1-\epsilon\), Lemma~\ref{lemma:bounded_delay_reduces_to_early_reflection}
shows that each possible earliest incorrect edge requires at most
\[
2B_{\Delta_r}L
\]
useful delayed reflections before useful reflections through that edge return
the earliest incorrect node itself. Once useful reflections return the earliest
incorrect node, the original node-suppression argument applies: after at most an
additional \(L\) returns of that earliest incorrect node, the corresponding
first incorrect edge is suppressed. Since \(B_{\Delta_r}\ge 1\), each possible
first incorrect edge therefore absorbs at most
\[
3B_{\Delta_r}L
\]
useful delayed reflections before being suppressed.

There are at most \(knW\) possible first incorrect edges. Hence, on the
no-return event, after at most
\[
N
:=
3knW B_{\Delta_r}L
\]
useful delayed reflections, all possible first incorrect edges are suppressed.
Consequently, any future incorrect trajectory within the horizon would have to
traverse a suppressed first incorrect edge, which does not occur on the
no-return event. Therefore, after \(N\) useful delayed reflections, either a
correct solution has already been sampled or no further incorrect trajectory is
sampled during the first \(\overline T\) rounds.

It remains to show that \(N\) useful delayed reflections occur quickly. For an
incorrect round \(t\), let \(v_k\) be the earliest incorrect node of the sampled
trajectory, and define
\[
Z_t
:=
\mathbf 1
\left\{
r_t\in
\{v_k,\dots,v_{\min(k+\Delta_r,n)}\}
\right\}.
\]
By the bounded useful-delay reflection assumption, on every eligible incorrect
round,
\[
\mathbb E[Z_t\mid \mathcal C_{t-1}]\ge p_r.
\]

As before, there are at most \(knWm_r\) ineligible incorrect rounds: each
possible earliest incorrect node can appear fewer than \(m_r\) times before it
becomes eligible, and there are at most \(knW\) possible earliest incorrect
nodes.

Applying the same Azuma--Hoeffding argument as in the proof of
Theorem~\ref{result:inference_backtracks}, with this new definition of \(Z_t\),
we obtain that after
\[
R
\ge
\frac{2}{p_r}
\left(
N+\frac{4}{p_r}\log\frac1\delta
\right)
\]
eligible incorrect rounds, at least \(N\) useful delayed reflections have
occurred with probability at least \(1-\delta\).

Therefore it suffices to take
\[
T
\ge
knWm_r
+
\frac{2}{p_r}
\left(
N+\frac{4}{p_r}\log\frac1\delta
\right).
\]
Substituting
\[
N
=
3knW B_{\Delta_r}
\frac{1}{\eta}
\log\frac{\overline T\,knW B_{\Delta_r}}{\epsilon\gamma}
\]
gives
\[
T
=
\widetilde O\!\left(
knWm_r
+
\frac{knW B_{\Delta_r}}{p_r\eta}
\log\frac{\overline T\,knW B_{\Delta_r}}{\epsilon\gamma}
+
\frac{1}{p_r^2}\log\frac1\delta
\right).
\]

Combining the concentration event with the no-return event and taking
\(\epsilon=\delta/2\) gives success probability at least \(1-\delta\). Finally,
the same self-consistency argument used in the proof of
Theorem~\ref{result:inference_backtracks} removes the dependence on
\(\overline T\), since \(\overline T\) appears only logarithmically. Thus, with
probability at least \(1-\delta\), a correct solution is sampled after at most
\[
T
=
\widetilde O\!\left(
knWm_r
+
\frac{knW B_{\Delta_r}}{p_r\eta}\log\frac1\gamma
+
\frac{1}{p_r^2}\log\frac1\delta
\right)
\]
rounds.

\qed

\subsection{Synthetic problems}\label{empirical:synthetic}

We validate the qualitative predictions of our theory in a fully controlled synthetic setting, where the search space, reflection signal, and posterior updates are explicitly implemented. We generate random search trees of fixed width $W$ and varying depth $n$, with edge weights inducing a prior distribution over paths; a single correct solution is chosen by sampling a path from this prior. We implement the posterior update rule of Definition \ref{def:in-context_learning} directly and control the reflection signal to test different feedback regimes.

\paragraph{Efficiency of posterior updates:}
The corollary from Theorem\ref{result:inference_backtracks} predicts that independent sampling from the prior requires exponentially many samples to find the correct path as depth increases, whereas sequential sampling with posterior updates requires only polynomially many attempts. Figure \ref{fig:synthetic_expressivity}(a) confirms this separation: parallel sampling exhibits linear growth in the log domain (exponential complexity), while sequential sampling with posterior updates grows sublinearly.

\paragraph{Failure with late reflections:}Proposition \ref{res:late_reflections} predicts that this advantage disappears if reflection fails to identify early mistakes. Figure 2(b) shows that when reflections are delayed to depths $\geq n/2$, in-context search again exhibits exponential complexity, isolating early error localization as the key mechanism enabling efficiency.
\begin{figure}[h!]
    \centering
    \includegraphics[width=0.8\linewidth]{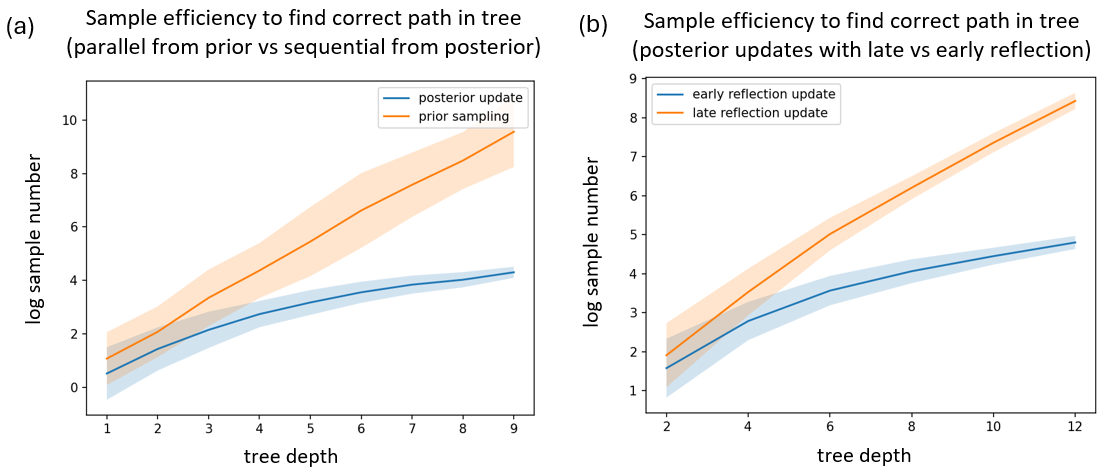}
    \caption{(a) Comparison of sample efficiency to find the correct path in trees of increasing depth, when parallel sampling from the prior vs sequential sampling with posterior updates. (b) Comparison of sample efficiency to find the correct path in trees of increasing depth, in sequential sampling with posterior updates from late mistake reflections vs early mistake reflections.}
    \label{fig:synthetic_expressivity}
\end{figure}

\section{Reasoning Loops as a Failure Mode}\label{sec:loops}

Assumption~\ref{assumption:reachability} prevents the model from repeatedly restarting from prefixes that have already been identified as incorrect.

While this assumption is not strictly necessary for the success of in-context search, removing it permits pathological behaviors in which the model cycles between previously invalidated prefixes without making progress. We refer to such behaviors as \emph{reasoning loops}. These loops do not necessarily prevent success in all cases, but they can significantly degrade efficiency by repeatedly revisiting the same incorrect regions of the search space. Hence we do not expect them to be part of the efficient search regime.

To assess how often this occurs in practice, we analyze failed reasoning trajectories on the AIME 2025 benchmark. For each model, we sample chains-of-thought (CoTs) that do not reach a correct solution and measure the fraction that exhibit a reasoning loop, defined as repeated resumption from a previously invalidated prefix.

\begin{center}
\begin{tabular}{lcc}
\toprule
Model & Greedy & Nucleus \\
\midrule
DeepSeek-R1-Distill-Qwen-2.5-1.5B & 0.83 & 0.48 \\
DeepSeek-R1-Distill-Qwen-2.5-7B   & 0.52 & 0.38 \\
DeepSeek-R1-Distill-Qwen-2.5-14B  & 0.44 & 0.28 \\
\bottomrule
\end{tabular}
\end{center}

We observe that a substantial fraction of failed trajectories exhibit looping behavior, both with greedy decoding and nucleus sampling. 

This supports the relevance of the reachability assumption as a mechanism for ruling out inefficient dynamics. We also observe that larger models exhibit fewer loops, suggesting improved implicit avoidance of previously invalidated prefixes.

Thus, Assumption~\ref{assumption:reachability} should be viewed not as a claim that real models never loop, but as an idealized condition isolating the regime where search is efficient.

\section{Sequential Reflection Rounds}\label{sec:sequential_revisions}

We complement the prefix-conditioned pass-rate analysis with an explicit sequential reflection setting. In each round, the model is prompted to produce a final answer, after which generation is resumed with a self-reflection trigger. Pass rate is then estimated after each revised solution attempt.

To induce self-reflection, we append the beginning-of-CoT token followed by the explicit reflection phrase ``\texttt{<think>Wait}''. To terminate the reflection segment and elicit a new answer, we append ``\texttt{Summing up, final answer</think>}''. Pass rates are estimated from 20 sampled completions conditioned on the resulting prefix. Each answer-generation phase is allocated a 1k-token budget, while reflection phases use either a 100-token or 1k-token budget. This setup is not intended to isolate a pure reflection oracle, since the reflection segment may also perform additional reasoning; rather, it probes whether an explicit self-reflection trigger can substantially change the distribution over final answers.

Figure~\ref{fig:sequential_revisions} shows that sequential reflection can sometimes induce large pass-rate increases, from below $0.05$ (no successes among 20 evaluation samples) to near $1.0$ within a few rounds. At the same time, reflections can also reduce pass rate or leave it unchanged, indicating that the process is not monotonic improvement but redistribution of probability mass across competing continuations. Similar to the prefix-conditioned trajectory analysis in the main text, the resulting dynamics are sparse and highly heterogeneous.

\begin{figure}[h!]
    \centering
    \includegraphics[width=0.9\linewidth]{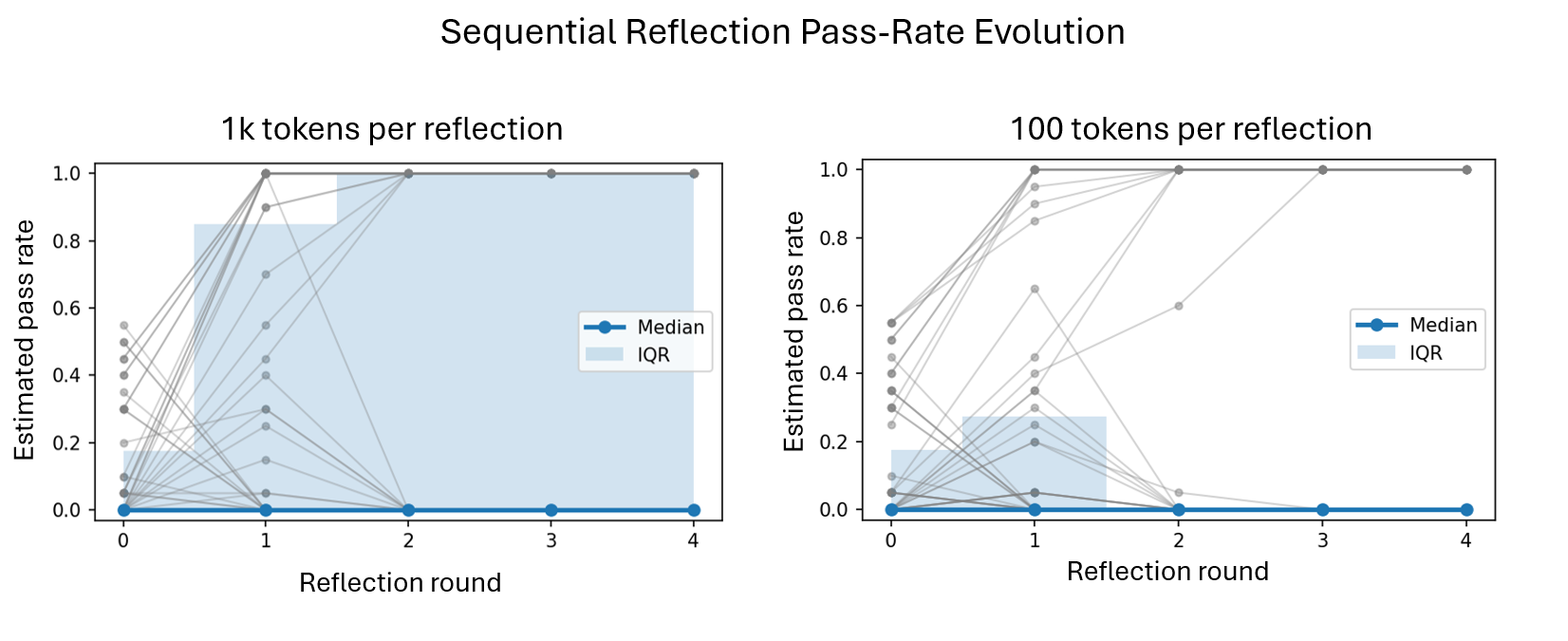}
    \caption{
    Pass-rate evolution under explicit sequential self-reflection rounds on AIME 2025. Gray lines show individual trajectories; blue shows the median with interquartile range. Left: 1k-token reflection budget. Right: 100-token reflection budget. Reflection can sometimes induce large positive shifts in pass rate, while other trajectories remain unchanged or degrade.
    }
    \label{fig:sequential_revisions}
\end{figure}

\section{Evidence in Large Reasoning Models - Results in Complementary Model}\label{sec:complementary-7B}

In \ref{empirical:real}, we provided results for prefix-conditioned pass rates on DeepSeek-R1-Distill-Qwen-2.5-1.5B. We now provide complementary results for DeepSeek-R1-Distill-Qwen-2.5-7B. Figure \ref{fig:spaghetti-7B} shows qualitatively similar results to the $1.5B$ variant, with a larger tail of reflections leading to an increase in pass rate.

\begin{figure}[h!]
    \centering
    \includegraphics[width=0.9\linewidth]{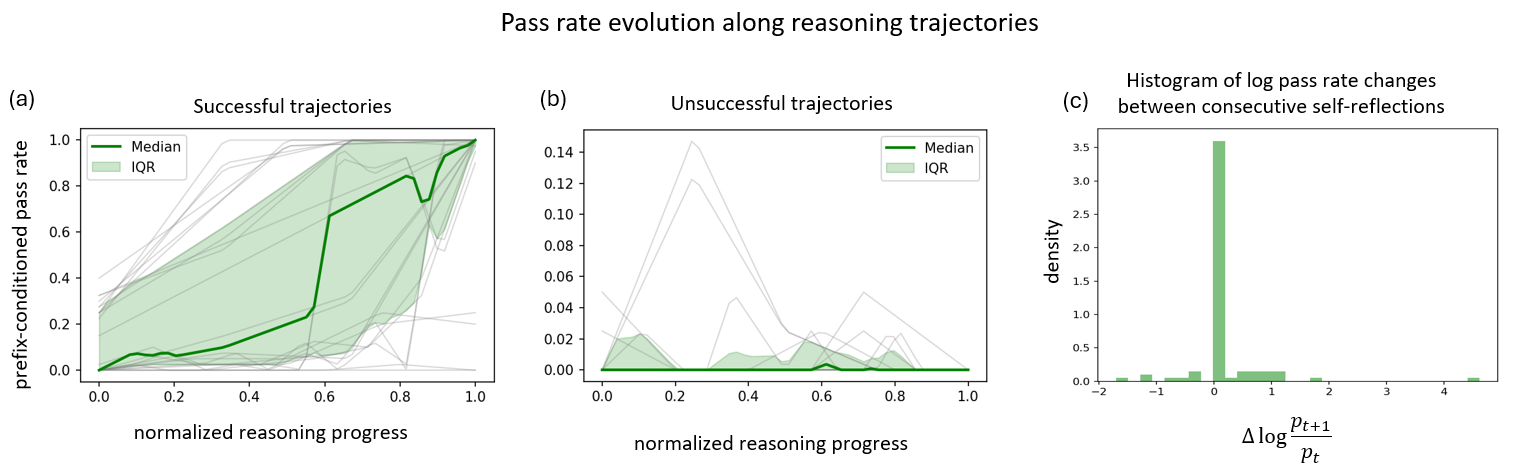}
    \caption{
    Evolution of prefix-conditioned pass rate along reasoning trajectories on AIME 2025 (DeepSeek-R1-Distill-Qwen-2.5-7B), using prefixes sampled at $\sim$1k-token intervals. Gray lines show individual trajectories; green shows the median with interquartile range. (a) Successful trajectories generally increase, with occasional reversals. (b) Unsuccessful trajectories remain low and highly non-monotonic. (c) Reflection-level log changes in pass rate along successful trajectories between consecutive self-correction segments.
    }
    \label{fig:spaghetti-7B}
\end{figure}

\section{Experimental details}

\paragraph{Reflection on arithmetic traces:} We start with a random number, then sample chains via a random choice between the actions $+,-,*,/$, with another random number. A single incorrect calculation is made along the chain, after which correct calculations propagate the error. The model is instructed to find the earliest error.

\paragraph{Reflection on equation solutions:} We start with an uncompressed linear equation written as a sum of $n$ first degree polynomials. Each step merges two of the polynomials into one, reducing to $n-1$ first degree polynomials. Such that after $n$ steps, the equation reduced. A single incorrect calculation is made along the chain, after which correct calculations propagate the error. The model is instructed to find the earliest error.

\paragraph{Evaluating pass rate along a CoT:} We sample full CoTs to AIME 2025 questions from a model (DeepSeek-R1-Distilled-1.5B/7B). The sampling configuration is set to 16k new tokens, using nucleus sampling with default parameters of the model. We then truncate the CoT at evenly spaced positions of $1k$ tokens (rounding up at ends of sentence) along the chain and evaluate the pass rate at those positions - by instructing the model to provide a final answer, followed by the model's end of CoT token $</ think>$, with a budget of up to $1024$ tokens and same configurations, using $100$ generations for the pass rate evaluation. Additionally, for the same evaluation at self-reflection points, we truncate at phrases that imply explicit self-reflection ("Wait").

\paragraph{Evaluation of change in log likelihood of pass rate:} In figure \ref{fig:spaghetti}(c), we show a histogram of change in pass rate log ratio between consecutive self-reflections. As the pass rate is evaluated with $100$ final answer completions, pass rates below $0.01$ are not captured. We thus use a truncated version - $\log\frac{p_t+0.01}{p_{t-1}+0.01}$.

\paragraph{Synthetic trees:} We sampled trees of width $3$ and depth $n$. Posterior updates used $\eta=0.5$.

\paragraph{Compute:} All experiments were performed with up to $4$ $L4$ gpus.

\end{document}